\documentclass{article}

\usepackage{microtype}
\usepackage{graphicx}
\usepackage{subfigure}
\usepackage{booktabs} %

\usepackage{hyperref}

 \usepackage[accepted]{icml2024}

\usepackage{amsmath}
\usepackage{amssymb}
\usepackage{mathtools}
\usepackage{amsthm}

\usepackage[capitalize,noabbrev]{cleveref}

\DeclareMathOperator*\lowlim{\underline{lim}}
\DeclareMathOperator*\uplim{\overline{lim}}

\usepackage[textsize=tiny]{todonotes}

\theoremstyle{plain}
\newtheorem{theorem}{Theorem}[section]
\newtheorem{proposition}[theorem]{Proposition}
\newtheorem{lemma}[theorem]{Lemma}

\theoremstyle{definition}
\newtheorem{definition}[theorem]{Definition}

\theoremstyle{remark}

\makeatletter
\newtheorem*{rep@theorem}{\rep@title}
\newcommand{\newreptheorem}[2]{%
	\newenvironment{rep#1}[1]{%
		\def\rep@title{#2 \ref{##1}}%
		\begin{rep@theorem}}%
		{\end{rep@theorem}}}
\makeatother

\newreptheorem{theorem}{Theorem}
\newreptheorem{lemma}{Lemma}
\newreptheorem{proposition}{Proposition}
\newreptheorem{corollary}{Corollary}
\newreptheorem{claim}{Claim}
\newreptheorem{assumption}{Assumption}
\newreptheorem{conjecture}{Conjecture}

\newcommand{\ep}{\hfill $\Box$}

\DeclareMathOperator*{\argmax}{argmax} 
\DeclareMathOperator*{\argmin}{argmin}

\renewcommand{\tilde}{\widetilde}
\renewcommand{\star}{*}

\newcommand{\PP}{\mathbb{P}}
\newcommand{\EE}{\mathbb{E}}
\newcommand{\NN}{\mathbb{N}}
\newcommand{\RR}{\mathbb{R}}
\newcommand{\bmu}{\boldsymbol{\mu}}
\newcommand{\bxi}{\boldsymbol{\xi}}
\newcommand{\bl}{\boldsymbol{\lambda}}
\newcommand{\bp}{\boldsymbol{\pi}}
\newcommand{\bet}{\boldsymbol{\eta}}
\newcommand{\olog}{\overline{\log}}
\newcommand{\zz}{\boldsymbol{z}}
\newcommand{\yy}{\boldsymbol{y}}
\newcommand{\xx}{\boldsymbol{x}}

\newcommand{\Ij}{{\cal I}_j}

\usepackage{bbm}

\usepackage{color}
\usepackage{dirtytalk}
\renewcommand{\cite}{\citep}
\renewcommand{\citet}{\citep}

\renewcommand{\limsup}{\uplim}
\renewcommand{\liminf}{\lowlim}

\icmltitlerunning{On Universally Optimal Algorithms for A/B Testing}

\begin{document}

\twocolumn[
\icmltitle{On Universally Optimal Algorithms for A/B Testing}

\icmlsetsymbol{equal}{*}

\begin{icmlauthorlist}
\icmlsetsymbol{intern}{\S}
\icmlauthor{Po-An Wang}{intern,KTH,CyberAgent}
\icmlauthor{Kaito Ariu}{CyberAgent}
\icmlauthor{Alexandre Proutiere}{KTH}
\end{icmlauthorlist}
 \icmlaffiliation{KTH}{EECS and Digital Futures, KTH, Stockholm, Sweden}
 \icmlaffiliation{CyberAgent}{CyberAgent, Tokyo, Japan}

 \icmlcorrespondingauthor{Po-An Wang}{wang9@kth.se}

\icmlkeywords{Machine Learning, ICML}

\vskip 0.3in
]

\printAffiliationsAndNotice{\S Work initiated during the internship at CyberAgent.}

\begin{abstract}
We study the problem of best-arm identification with fixed budget in stochastic multi-armed bandits with Bernoulli rewards. For the problem with two arms, also known as the A/B testing problem, we prove that there is no algorithm that (i) performs as well as the algorithm sampling each arm equally (referred to as the {\it uniform sampling} algorithm) in all instances, and that (ii) strictly outperforms uniform sampling on at least one instance. In short, there is no algorithm better than the uniform sampling algorithm. To establish this result, we first introduce the natural class of {\it consistent} and {\it stable} algorithms, and show that any algorithm that performs as well as the uniform sampling algorithm in all instances belongs to this class. The proof then proceeds by deriving a lower bound on the error rate satisfied by any consistent and stable algorithm, and by showing that the uniform sampling algorithm matches this lower bound. Our results provide a solution to the two open problems presented in \citep{qin2022open}. 
For the general problem with more than two arms, we provide a first set of results. We characterize the asymptotic error rate of the celebrated Successive Rejects (SR) algorithm \citep{audibert2010best} and show that, surprisingly, the uniform sampling algorithm outperforms the SR algorithm in some instances.
\end{abstract}

\section{Introduction}
We study the problem of Fixed-Budget Best-Arm Identification (FB-BAI) in stochastic multi-armed bandits with Bernoulli rewards. In this problem, the learner sequentially pulls an arm and observes a random reward generated according to the corresponding distribution. The expected rewards of the arms are initially unknown. The learner has a fixed budget of $T\in \mathbb{N}$ pulls or samples, and after gathering these samples, she has to return what she believes to be the arm with the highest mean reward. For any  $k\in [K]\coloneqq\{1,\ldots,K\}$, we denote by $\mu_k\in (0,1)$ the unknown mean reward of arm $k$. We assume that the best arm is unique and define the parameter set of the mean rewards as $\Lambda=\{\bmu\in (0,1)^K:\exists k: \mu_k> \mu_j, \forall j\neq k\}$. A strategy for fixed-budget best-arm identification consists of a {\it sampling rule} and a {\it decision rule}. The sampling rule determines the arm $A_t \in [K]$ to be explored in round $t$, based on past observations. The corresponding observed reward is $X_{t}\in \{0,1\}$. The arm $A_t$ selected in round $t$ is ${\cal F}_{t}$ measurable where ${\cal F}_{t}$ demotes the $\sigma$-algebra generated by the set of random variables $\{A_1, X_{1}, \ldots, A_{t-1}, X_{t-1}\}$. After $T$ rounds, the decision rule returns an answer $\hat{\imath} \in [K]$, which is ${\cal F}_{T}$ measurable. The goal is to identifiy a strategy that minimizes the error probability defined as
$$
p_{\bmu,T}\coloneqq\PP_{\bmu}\left[\hat{\imath}\neq 1(\bmu)\right],
$$
where $1(\bmu)\coloneqq\arg\max_k\mu_k$ denotes the unique best arm under the instance $\bmu$. 

A naive strategy consists in allocating a fixed fraction of the budget to sample each arm. 
Once the budget is exhausted, the strategy then returns the arm with the highest empirical mean.
We refer to such a strategy as a {\it static} algorithm (in contrast to adaptive algorithms that may select the arm to pull next based on the rewards observed so far). An example of a static strategy is the uniform sampling strategy that allocates the budget fairly among arms. Static algorithms are well-understood and in particular, their asymptotic error rates are known \citep{glynn2004large}. Many adaptive sampling algorithms have been designed, see, e.g., \cite{audibert2010best,gabillon2012best,karnin2013almost,russo2020simple,komiyama2022minimax,wang2023best}, with the hope of an improved performance compared to static algorithms. It is still unclear whether this hope can actually be fulfilled. 

Despite recent research efforts, the FB-BAI problem remains largely open \citep{qin2022open}. This contrasts with the two other classical learning tasks in stochastic bandits, namely regret minimization \citep{lai1985asymptotically} and best arm identification with fixed confidence \citep{garivier2016optimal}. Indeed, for these tasks, asymptotic instance-specific performance limits and matching algorithms have been derived. In this paper, we aim at improving our understanding of the FB-BAI problem and more specifically at answering the following two natural questions, mentioned as open problems in \cite{qin2022open}.

{\it Open problem 1.} Is there an algorithm whose performance is as good as that of the uniform sampling algorithm on all instances and that strictly outperforms the latter on some instances?

{\it Open problem 2.} Can we derive an asymptotic and instance-specific error rate lower bound that (i) is satisfied by all algorithms within a wide class ${\cal A}$ of algorithms and that (ii) is achieved by a single algorithm in ${\cal A}$ on all instances? 

We address both open problems in the case of the FB-BAI problem with two arms (also referred to as the A/B testing problem) with Bernoulli rewards. We also provide a first set of results towards addressing these problems in the general setting with more than two arms. More precisely our contributions are as follows.

{\bf Contributions.}

(a) For the A/B testing problem, we prove that there is no algorithm strictly outperforming the uniform sampling algorithm (Theorem~\ref{thm:main}). To this aim, we first introduce the natural class of consistent and stable algorithms (stability here just means that the algorithm exhibits a symmetric and continuous behavior with respect to the instances). We then show that this class includes any algorithm performing as well as the uniform sampling algorithm on all instances (Theorem~\ref{thm:uni}). We finally derive an instance-specific lower bound on the error rate satisfied by any consistent and stable (Theorem~\ref{thm:stable}). As it turns out, this lower bound corresponds to the performance of the uniform sampling algorithm. The answer to the question of the open problem 1 is hence negative. 

(b) Our analysis further provides a positive answer to the question of the open problem 2. Indeed, it yields an instance-specific error rate lower bound for the class of consistent and stable algorithms, and the performance of the uniform sampling algorithm matches this fundamental limit. 

(c) For the FB-BAI problem with more than two arms, we manage to exactly characterize the asymptotic error rate of the celebrated Successive Rejects (SR) algorithm (Audibert et al., 2010) (Theorem~\ref{thm:tight_analysis_SR}). This contrasts with existing analyses of adaptive algorithms where only upper bounds of the error rate can be derived. Using this characterization, we show that, surprisingly, the uniform sampling algorithm outperforms the SR algorithm in certain instances (Theorem~\ref{thm:sr}).

\section{Preliminaries and Main Result}

In this section, we first present existing results on the performance of static algorithms in the A/B testing problem. We then state our main result: there is no strictly better algorithm than the uniform sampling algorithm.  

\subsection{Performance of Static Algorithms}

In two-arm bandits, a static algorithm is parameterized by a single variable $x\in (0,1)$ specifying the fraction of the budget used to sample the second arm (a static algorithm parameterized by $x$ pulls the first arm $(1-x)T+o(T)$ times and pulls the second arm $xT+o(T)$ times).
Defining 
\begin{equation}\label{eq:g}
    g(x,\bmu)\coloneqq\min_{\lambda\in [0,1] } (1-x)d(\lambda ,\mu_1)+x d(\lambda,\mu_2),
\end{equation}
where $d(a,b)$ is the KL-divergence between two Bernoulli distributions with respective means $a$ and $b$, \citet{glynn2004large} shows that under a static algorithm parametrized by~$x$, 
$$
\lim_{T\rightarrow\infty} \frac{T}{\log(1/p_{\bmu,T})}=\frac{1}{g(x,\bmu)}.
$$
The optimization problem in \eqref{eq:g} can be solved by explicitly writing the KKT conditions \citep{glynn2004large}. Its unique solution and value are given by
\begin{align}
\lambda(x,\bmu) &=\frac{(\frac{\mu_1}{1-\mu_1})^{1-x}(\frac{\mu_2}{1-\mu_2})^x}{1+(\frac{\mu_1}{1-\mu_1})^{1-x}(\frac{\mu_2}{1-\mu_2})^x} \in (0,1),\label{eq:lambda} \\
g(x,\bmu) & =-\log ((1-\mu_1)^{1-x}(1-\mu_2)^x+\mu_1^{1-x}\mu_2^x).\label{eq:g alt}
\end{align}
From (\ref{eq:g alt}), we can readily verify that $g(x,\bmu)$ is strictly concave in $x$, i.e., $\frac{\partial^2g(x,\bmu)}{\partial^2 x}<0$ as long as $\mu_1\neq \mu_2$ or equivalently $\bmu\in\Lambda$.\footnote{For completeness, we include proof of \eqref{eq:lambda}, \eqref{eq:g alt}, and the strict concavity of $g(x, \bmu)$ in Appendix~\ref{app:dual}.} Therefore, $g(x,\bmu)$ has a unique maximizer denoted by $x^\star(\bmu)\coloneqq\argmax_x g(x,\bmu)$. 
Given the expected rewards $\bmu$ of the arms, $x^\star(\bmu)$ corresponds to the static algorithm with the best possible performance. However, under this static algorithm, the fraction of the budget used for each arm depends on the initially unknown $\bmu$. 

Over the last few years, researchers have tried to determine whether there exists an adaptive algorithm that could achieve the performance of the best static algorithm for any $\bmu$. The answer to this question is actually negative, as recently proved in \citet{degenne2023existence}: for any algorithm, there exists an instance $\bmu$ such that the considered algorithm performs strictly worse than the best static algorithm on this instance. Refer to Section~\ref{sec:related} for additional details. This negative result illustrates the difficulty of devising adaptive and efficient algorithms. We establish even more striking evidence of this challenge. We show that there is no algorithm universally outperforming the uniform sampling algorithm. We formalize this result below. 

\subsection{Main Result}

The performance of the uniform sampling algorithm is characterized by $g(1/2,\bmu)$. More precisely, under this algorithm, 
$$
\forall \bmu\in \Lambda, \quad\lim_{T\rightarrow\infty} \frac{T}{\log(1/p_{\bmu,T})}=\frac{1}{g(1/2,\bmu)}.
$$ 

Our main result concerns the class of {\it better than uniform} algorithms already introduced and discussed in \citet{qin2022open}. These algorithms are at least as good as the uniform sampling algorithm in all instances.

\begin{definition}\label{def:uni}
    An algorithm is {\it better than uniform} if
     $$
    \forall \bmu\in \Lambda, \quad \uplim_{T\rightarrow \infty} \frac{T}{\log (1/p_{\bmu,T})}\le \frac{1}{g(1/2,\bmu)}.
    $$
\end{definition}
\begin{theorem}\label{thm:main}
    For any better than uniform algorithm, %
    $$
     \forall \bmu\in  \Lambda,\quad \lim_{T\rightarrow \infty} \frac{T}{\log (1/p_{\bmu,T})}=\frac{1}{g(1/2,\bmu)}.
    $$
\end{theorem}
\noindent
As a consequence, surprisingly, one cannot devise an adaptive algorithm that performs as well as the uniform sampling algorithm on all instances and that strictly outperforms it on some instances. This also implies that if an algorithm strictly outperforms the uniform sampling algorithm in at least one instance, then there is an instance where the uniform sampling algorithm strictly outperforms this algorithm. This provides a solution to the open problem~1 \citet{qin2022open} presented in the introduction (refer to Section~\ref{subsec:related_open_qin} for more details).

Theorem \ref{thm:main} is proved by combining the results presented in Sections \ref{sec:CS} and \ref{sec:LB}. There, we introduce the class of consistent and stable algorithms, and show that better than uniform algorithms are consistent and stable. We also establish the error rate achieved by the uniform sampling algorithm constitutes an error rate lower bound satisfied by consistent and stable algorithms. Note that these intermediate results towards Theorem \ref{thm:main} provide a solution to the open problem~2 \citet{qin2022open} presented in the introduction.

\subsection{Notation} For each $t\in \{1,2,\ldots, T\}$ and $k\in \{1,2\}$, define $N_{k}(t)\coloneqq \sum_{s=1}^{t}\mathbbm{1}\{A_s=k\}$ as the number of times arm $k$ is pulled up to round $t$, and $\omega_k(t)\coloneqq N_{k}(t)/t$ as the proportion of times arm $k$ is pulled.

\section{Stable and Consistent Algorithms}\label{sec:CS}

In this section, we demonstrate that any better than uniform algorithm is both consistent and stable, as defined below.
\begin{definition}\label{def:consistent}
    An algorithm is {\it consistent} if for all $\bmu\in \Lambda$, $\lim_{T\to\infty}p_{\bmu,T}=0$.
\end{definition}
\medskip

\begin{definition}\label{def:stable}
    An algorithm is {\it stable} if for any $a\in (0,1)$, the following properties hold:\\ 
    \noindent
    \textnormal{(A)} There exists $\{\bl^{(n)}\}_{n=1}^\infty\subset \{\bl\in \Lambda:\lambda_1>\lambda_2\}$  such that $\bl^{(n)}\xrightarrow{n\rightarrow\infty} (a,a)$ and
    $$
    \lim_{n\rightarrow \infty} \lowlim_{T\rightarrow \infty}\EE_{\bl^{(n)}}[\omega_2(T)] 
           = \lim_{n\rightarrow \infty} \uplim_{T\rightarrow \infty}\EE_{\bl^{(n)}}[\omega_2(T)] =\frac{1}{2}.
    $$
    \noindent
    \textnormal{(B)} There exists $\{\bp^{(n)}\}_{n=1}^\infty \subset \{\bp\in \Lambda:\pi_1<\pi_2\}$ such that $\bp^{(n)}\xrightarrow{n\rightarrow\infty} (a,a)$ and 
    $$
    \lim_{n\rightarrow \infty} \lowlim_{T\rightarrow \infty}\EE_{\bp^{(n)}}[\omega_2(T)]
           =
    \lim_{n\rightarrow \infty} \uplim_{T\rightarrow \infty}\EE_{\bp^{(n)}}[\omega_2(T)] 
    = \frac{1}{2}.
    $$
\end{definition}
\noindent
Intuitively, an algorithm is stable if it exhibits a symmetric and continuous behavior with respect to the bandit instances.
The notion of stability is natural and just refers to the property of evenly allocating the budget when the arms have very similar mean rewards. It is satisfied by the uniform sampling algorithm (and of course all the adaptive algorithms that evenly select arms in the case of two-armed bandits) and by most {\it reasonably adaptive} algorithms. We give several families of stable algorithms in Appendix \ref{app:eftt}. For example, stability is guaranteed as soon as the algorithm is designed such that the number of times arm 1 is sampled up to time $t$ closely matches $tf(\hat{\mu}_1(t), \hat{\mu}_2(t))$, where $\hat{\mu}_1(t)$ and $ \hat{\mu}_2(t)$ are the current estimates of the mean rewards and $f>0$ is a continuous function such that $f(a,a)=1/2$ for any $a \in (0,1)$.

In addition, as established in the following theorem, better than uniform algorithms are stable.

\begin{theorem}\label{thm:uni}
   A better than uniform algorithm is consistent and stable.
\end{theorem}

\subsection{Proof of Theorem \ref{thm:uni}}

{\it Consistency.} In view of Definition~\ref{def:uni}, a better than uniform algorithm is consistent. Indeed, for any $\bmu\in \Lambda$, there exists $T_{\bmu}\in \NN$ such that if $T > T_{\bmu}$, then 
$
\frac{T}{\log(1/p_{\bmu,T})}\le \frac{2}{g(1/2,\bmu)}
$.
As a result, we have
$p_{\bmu,T}\le e^{-\frac{1}{2}g(1/2,\bmu)T}.$ We conclude the proof by observing that $g(1/2,\bmu)>0$.

\medskip
\noindent
{\it Stability.} We show by contradiction that a better than uniform algorithm is stable. Suppose there exists $a\in (0,1)$ such that (B) in Definition \ref{def:stable} does not hold (if (A) in Definition~\ref{def:stable} does not hold, one can obtain a contradiction in a symmetric way). The following lemma is proven in Appendix~\ref{app:techlemmas}.

\begin{lemma}\label{lem:negation_stable}
    Let $a\in (0,1)$. Assume that the statement \textnormal{(B)} of Definition~\ref{def:stable} does not hold. Then for any $\{\bp^{(n)}\}_{n=1}^\infty\subset \{\bp\in \Lambda:\pi_1<\pi_2\}$ such that $\bp^{(n)}\xrightarrow{n\rightarrow\infty} (a,a)$, there exists a value $x \in [0, 1]$ and increasing sequences of integers $\{n_{m}\}_{m=1}^\infty, \{T_{m, \ell}\}_{\ell=1}^\infty \subset \mathbb{N}$ such that 
    $$
\lim_{m \to \infty}\lim_{\ell \to \infty}  \EE_{\bp^{(n_m)}}[\omega_2(T_{m,\ell})] =x \neq \frac{1}{2}.
$$
\end{lemma}
\noindent
Let $x$, $\{n_{m}\}_{m=1}^\infty$, and $ \{T_{m, \ell}\}_{\ell=1}^\infty$ be a real number and sequences satisfying the statement of Lemma \ref{lem:negation_stable}. 
Using a standard change-of-measure argument (e.g., inequality (6) in \cite{garivier2019explore}), for any $\bmu\in \Lambda$ with $\mu_1>\mu_2$, for each $m, \ell \in \mathbb{N}$, 
\begin{align}\label{eq:uni1}
 \nonumber   & \EE_{\bp^{(n_m)}}[N_1(T_{m, \ell})]d(\pi_1^{(n_m)},\mu_1)
 \\
\nonumber  & \ \ \  +\EE_{\bp^{(n_m)}}[N_2(T_{m,\ell})]d(\pi^{(n_m)}_{2},\mu_2)
 \\
 \nonumber &\ge d(\PP_{\bp^{(n_m)}}[\hat{\imath}=2],\PP_{\bmu}[\hat{\imath}=2])\\
    &\ge \PP_{\bp^{(n_m)}}[\hat{\imath}=2]\log\left(1/p_{\bmu,T_{m, \ell}}\right) -\log 2,
\end{align}
where the last inequality stems from the fact that $d(p,q)\ge p\log(1/q)-\log 2$ for all $p,q\in [0,1]$.
By dividing both sides of equation (\ref{eq:uni1}) by $T_{m, \ell}$, we can rearrange the inequality as follows:
\begin{align*}
 & \frac{T_{m, \ell}}{ \PP_{\bp^{(n_m)}}[\hat{\imath}=2] \log (1/p_{\bmu,T_{m, \ell}})}
 \\
 & \ge \bigg( \EE_{\bp^{(n_m)}}[\omega_1(T_{m,\ell})]d(\pi_1^{(n_m)},\mu_1)
 \\
 & \ \ \ \ \ + \EE_{\bp^{(n_m)}}[\omega_2(T_{m,\ell})]d(\pi_2^{(n_m)},\mu_2) + \frac{\log 2}{T_{m, \ell}}\bigg)^{-1}.
\end{align*}
Given that a better than uniform algorithm is consistent, it follows that $\PP_{\bp^{(n_m)}}[\hat{\imath}=2]=1-p_{\bp^{(n_m)},T_{m,\ell}}\xrightarrow{\ell\rightarrow\infty}1$ from $\pi_1^{(n_m)} < \pi_2^{(n_m)}$. 
Driving $\ell \to \infty$ first, and then letting $m \to \infty$, we obtain 
\begin{align}
 & \lowlim_{m \to \infty}\lowlim_{\ell \to \infty}\frac{T_{m, \ell}}{\log (1/p_{\bmu,T_{m, \ell}})} \nonumber
 \\
 & \ge  \frac{1}{(1-x)d(a,\mu_1)+xd(a,\mu_2)}.\label{eq:uni2}
\end{align}
Next, we use the following lemma related to the function $g$ and prove after completing the proof of Theorem~\ref{thm:uni}. Lemma~\ref{lem:prop g} is visualized in the left-hand side of Figure~\ref{fig:lem2_and_prop1}.
\begin{lemma}\label{lem:prop g}
For any $a\in (0,1)$, $x \in [0,1]$ such that $x \neq 1/2$, there exists $\bmu\in \Lambda$ such that $\mu_1>\mu_2,\,\lambda(x,\bmu)=a,$ and $g(x,\bmu)<g(1/2,\bmu)$. 
\end{lemma}
\begin{figure*}[ht]
\centering
\includegraphics[width=0.99\textwidth]{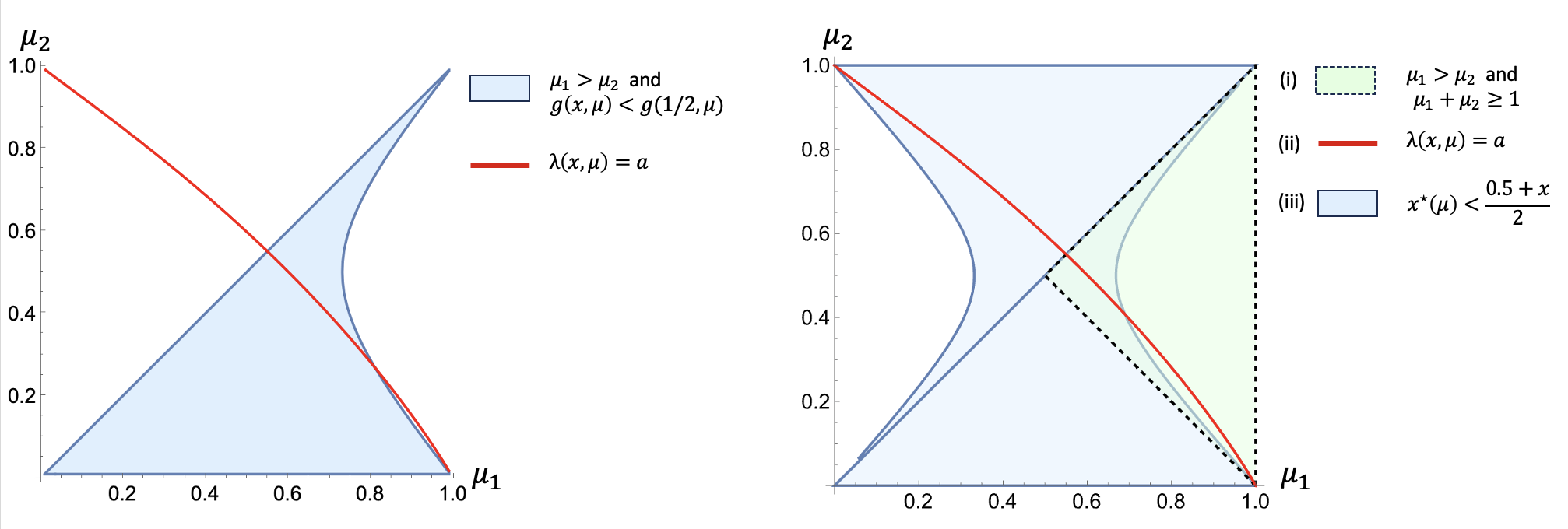}
\caption{Left: Visualization of Lemma~\ref{lem:prop g} with $a=0.55$ and $x=0.51$. The blue region indicates where $\mu_1>\mu_2$ and $g(x,\bmu)<g(1/2,\bmu)$. The red curve represents $\lambda(x,\bmu)=a$. The intersection of the blue region and red curve validates Lemma~\ref{lem:prop g}. Right: Visualization of Proposition~\ref{prop:intersection} with $a=0.55$ and $x=0.51$. The green region indicates \textnormal{(i)} $\mu_1>\mu_2,\,\mu_1+\mu_2\ge 1$. The red curve represents \textnormal{(ii)} $\lambda(x,\bmu)=a$. The blue region shows \textnormal{(iii)} $x^\star(\bmu)< (\frac{1}{2}+x)/2$.  The intersection of the three regions validates Proposition~\ref{prop:intersection}.}\label{fig:lem2_and_prop1}
\end{figure*}

\noindent
Plugging such $\bmu$ into (\ref{eq:uni2}) yields that 
\begin{align*}
     \liminf_{m \to \infty}\liminf_{\ell \to \infty}\frac{T_{m, \ell}}{\log (1/p_{\bmu,T_{m, \ell}})} \ge \frac{1}{g(x,\bmu)}
 >\frac{1}{g(1/2,\bmu)}.
\end{align*}
This contradicts the assumption that the algorithm is better than uniform.
\ep

\begin{proof}[Proof of Lemma~\ref{lem:prop g}]
We assume that $x \in (1/2, 1]$. The case for $x \in [0, 1/2)$ will be addressed at the end.
We first present Proposition~\ref{prop:intersection}, whose proof is given in Appendix~\ref{app:dual}, and its visualization is shown in the right panel of Figure~\ref{fig:lem2_and_prop1}.
\begin{proposition}\label{prop:intersection}
        For any $a\in (0,1),x\in (1/2,1]$, there exists an instance $\bmu\in \Lambda$ such that \textnormal{(i)} $\mu_1>\mu_2,\,\mu_1+\mu_2\ge 1$, \textnormal{(ii)} $\lambda(x,\bmu)=a$, and \textnormal{(iii)} $x^\star(\bmu)< (1/2+x)/2$.
\end{proposition}
Let $\bmu\in \Lambda$ be an instance satisfying the conditions of Proposition~\ref{prop:intersection}. 
If $x^\star(\bmu) \le 1/2$, the strict concavity of $g(\cdot,\bmu)$ immediately implies that $g(1/2,\bmu) > g(x,\bmu)$. On the other hand, if $x^\star(\bmu) > 1/2$, we can observe that $x^\star(\bmu) < (1/2 + x)/2 \le 3/4$, which leads to $\delta = x^\star(\bmu) - 1/2 \le \min\{x^\star(\bmu), 1 - x^\star(\bmu)\}$. 
We use the following proposition.
\begin{proposition}\label{prop:asymmetry}
      Suppose $\mu_1>\mu_2$ and $\mu_1+\mu_2\ge 1$. For any positive $\delta\le \min \{x^\star(\bmu),1-x^\star(\bmu)\},$
    $g(x^\star(\bmu)-\delta,\bmu)\ge g(x^\star(\bmu)+\delta,\bmu)$.
\end{proposition}
The proof and visualization of Proposition~\ref{prop:asymmetry} can be found in Appendix~\ref{app:asym} and Figure~\ref{fig:prop2}, respectively.  Setting $\delta = x^\star(\bmu) - 1/2$, we obtain the following inequality 
$$
g(1/2,\bmu) = g(x^\star(\bmu) - \delta,\bmu) \ge g(x^\star(\bmu) + \delta,\bmu).
$$
Using the strict concavity of $g(\cdot,\bmu)$ again and the fact that $x^\star(\bmu)+\delta=2x^\star(\bmu)-1/2<x$, we derive that $g(1/2,\bmu) \ge g(x^\star(\bmu)+\delta,\bmu)>g(x,\bmu)$. This concludes the proof when $x \in (1/2, 1]$.

Next, we consider the proof when $x \in [0,1/2)$. To this aim, we use the following symmetrical property of $g(x, \bmu)$.
\begin{proposition}\label{prop:reflect}
     Denote $\bar{\bmu}=(1-\mu_2,1-\mu_1).$
    For any $x\in (0,1)$, for any $\bmu\in \Lambda$, $g(1-x,\bar{\bmu})=g(x,\bmu)$ and $\lambda (1-x,\bar{\bmu})=1-\lambda (x,\bmu)$.
\end{proposition}
The proof of Proposition~\ref{prop:reflect} is presented in Appendix~\ref{app:reflect}.
The previous proof (replacing $x$ with $1-x \in (1/2, 1]$) yields the existence of $\bmu \in \Lambda$ such that $g(1-x,\bmu)<g(1/2, \bmu)$ and $\lambda(1-x,\bmu)=1-a$. Let $\bar{\bmu}=(1-\mu_2,1-\mu_1)$, Proposition~\ref{prop:reflect} and the strict concavity of $g(\cdot,\bmu)$ imply that
$$
g(x,\bar{\bmu})=g(1-x,\bmu)<g(1/2, \bmu)=g(1/2,\bar{\bmu}),
$$
and
$\lambda(x,\bar{\bmu})=1-\lambda(1-x,\bmu)=a.$ This concludes the proof for $x \in [0, 1/2)$, thus completing the proof of Lemma~\ref{lem:prop g}.
\end{proof}

\begin{figure*}[ht]
\centering
\includegraphics[width=0.8\textwidth]{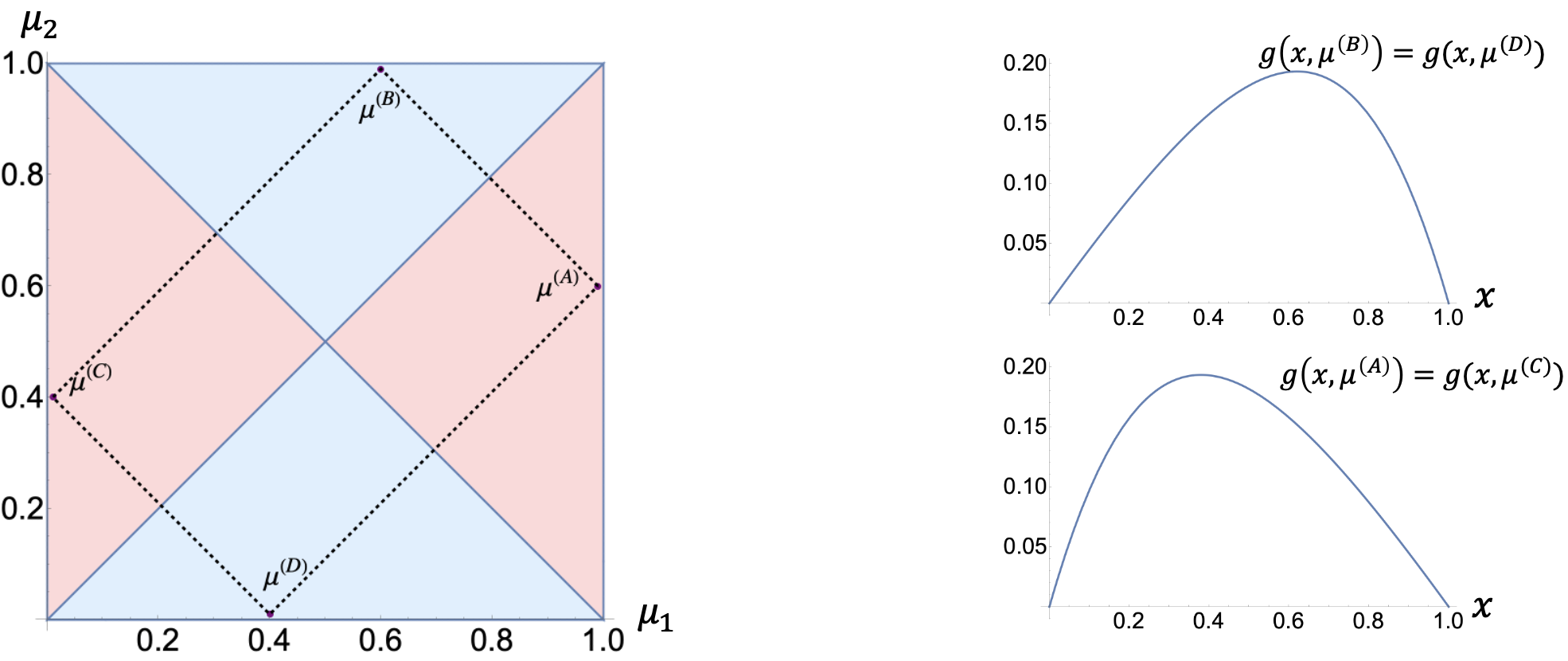}
\caption{
Visualization of the function $g(x, \bmu)$ properties. The left panel shows the partition of $\Lambda$ into four regions by $\mu_1=\mu_2$ and $\mu_1+\mu_2=1$, with blue indicating $ x^\star(\bmu)<1/2$ and red indicating $ x^\star(\bmu)>1/2$. Four bandit instances are chosen symmetrically from these regions for further analysis. The right panels show the functions $g(x,\bmu^{(A)})=g(x,\bmu^{(C)})$ (top) and $g(x,\bmu^{(B)})=g(x,\bmu^{(D)})$ (bottom), demonstrating the asymmetrical property as stated in Proposition~\ref{prop:asymmetry}.
}\label{fig:prop2}
\end{figure*}

\section{Error Rate of Consistent and Stable Algorithms}\label{sec:LB}

In this section, we establish that the performance of any stable and consistent algorithm is either equivalent to or worse than that of uniform sampling. 

\begin{theorem}\label{thm:stable}
    If an algorithm is consistent and stable, then 
    $$
   \forall \bmu\in \Lambda,\quad \liminf_{T\rightarrow \infty} \frac{T}{\log (1/p_{\bmu ,T})}  \ge \frac{1}{g(1/2,\bmu)}.
    $$
\end{theorem}

\begin{proof}
     Without loss of generality, assume $1(\bmu)=1$, namely, $\mu_1>\mu_2$. For any $\bp \in \Lambda$ such that $\pi_1 < \pi_2$, applying a standard change-of-measure argument, as in the proof of Theorem \ref{thm:uni}, yields that
\begin{align}\label{neq:stable1}
  \nonumber  & \EE_{\bp}[N_1(T)]d(\pi_1,\mu_1)+\EE_{\bp}[N_2(T)]d(\pi_{2},\mu_2)
 \\
 \nonumber  &\ge d(\PP_{\bp}[\hat{\imath}=2],\PP_{\bmu}[\hat{\imath}=2])
 \\
    &\ge \PP_{\bp}[\hat{\imath}=2]\log\left(1/p_{\bmu,T}\right) -\log 2.
\end{align}
Since $\pi_1<\pi_2$, the consistent assumption yields that $\PP_{\bp}[\hat{\imath}=2]=1-p_{\bp,T}\xrightarrow{T\rightarrow\infty}1$. Dividing the both sides of (\ref{neq:stable1}) by $T$ and taking $T\rightarrow\infty$, we obtain
\begin{align}
& \liminf_{T\rightarrow \infty} \frac{T}{\log (1/p_{\bmu ,T})}  \label{neq:stable2}
\\
\nonumber & \ge \liminf_{T\rightarrow \infty} \frac{1}{  \EE_{\bp}[\omega_1(T)] d(\pi_1,\mu_1)+ \EE_{\bp}[\omega_2(T)]d(\pi_2,\mu_2) }
\end{align}
by simply rearranging the terms. Next, we let $a=\lambda(\frac{1}{2},\bmu)$. Since the algorithm is stable, there exists $\{\bp^{(n)}\}_{n=1}^\infty\subset \Lambda$ such that $\pi^{(n)}_1<\pi^{(n)}_2$, $\bp^{(n)}\xrightarrow{n\rightarrow\infty}(a,a)$, and 
\begin{align*}
      \lim_{n\rightarrow \infty} \liminf_{T\rightarrow \infty}\EE_{\bp^{(n)}}[\omega_2(T)] &=\lim_{n\rightarrow \infty} \limsup_{T\rightarrow \infty}\EE_{\bp^{(n)}}[\omega_2(T)]  =\frac{1}{2}. 
\end{align*}
Notice that previous equations also imply that 
\begin{align}
   \lim_{n\rightarrow \infty} \lowlim_{T\rightarrow \infty}\EE_{\bp^{(n)}}[\omega_2(T)]&=\lim_{n\rightarrow \infty} \lowlim_{T\rightarrow \infty}(1-\EE_{\bp^{(n)}}[\omega_1(T)])\nonumber
   \\
   &=1-\lim_{n\rightarrow \infty} \uplim_{T\rightarrow \infty}\EE_{\bp^{(n)}}[\omega_1(T)]\nonumber
   \\
   &= \frac{1}{2}.\label{neq:stable4}
\end{align}
Thus, rearranging \eqref{neq:stable4} yields that 
\begin{equation*}%
    \lim_{n\rightarrow \infty} \limsup_{T\rightarrow \infty}\EE_{\bp^{(n)}}[\omega_1(T)] = \frac{1}{2}.
\end{equation*}
\noindent
Plugging them into (\ref{neq:stable2}) and taking $n$ to infinity yields that 

\begin{align*}
    \liminf_{T\rightarrow \infty} \frac{T}{\log (1/p_{\bmu ,T})}&\ge  \lim_{n\to\infty}(A_n)^{-1}
    =\Big(\lim_{n\to\infty}A_n\Big)^{-1}
    \\
    &\ge \frac{2}{d(a,\mu_1)+d(a,\mu_2)}= \frac{1}{g(1/2,\bmu)}.
\end{align*}
where 
$A_n$ is the limit superior of the sequence
$$
\big( \EE_{\bp^{(n)}}[\omega_1(T)] d(\pi_1^{(n)},\mu_1)
    + \EE_{\bp^{(n)}}[\omega_2(T)]d(\pi_2^{(n)},\mu_2)\big)_{T\in \NN}.
$$
\end{proof}

We remark that the combination of Theorems~\ref{thm:uni} and \ref{thm:stable} leads to Theorem~\ref{thm:main}.

\section{$K$-Armed Bandits with $K\ge 3$}

Extending our results to the general case where there are more than two arms is challenging. We investigate whether existing adaptive algorithms could be better than uniform algorithms. This question is not easy to answer because existing analyses of these algorithms provide {\it upper} bounds only on their error rates. Even if these upper bounds are, for some instances, worse than the error rate of the uniform sampling algorithm, it does not imply that the latter performs better on these instances. To answer the question, we also need to derive {\it lower} bounds on their error rates (which is challenging -- refer to \cite{wang2023best} for a detailed discussion).

In this section, we restrict our attention to the celebrated Successive Rejects (SR) algorithm \citep{audibert2010best}, and we manage to derive the exact expression of its asymptotic error rate. From there, we exhibit instances where surprisingly, the uniform sampling algorithm provably outperforms the SR algorithm. 

\subsection{The Successive Rejects Algorithm}

\begin{algorithm}
\caption{Successive Rejects (SR)}
\begin{algorithmic}\label{alg:SR}
\STATE {\bf} Initialization $\mathcal{C}_K\leftarrow [K]$, $j\leftarrow K$;
\FOR{$t=1,2,\ldots, T$}
    \IF{($j>2$ \textrm{and} $\min_{k\in \mathcal{C}_j} N_k(t)\ge \frac{T}{j\olog K}$)}
    \STATE $\ell_j\leftarrow \argmin_{k\in \mathcal{C}_j} \hat{\mu}_k(t)$ (tie broken arbitrarily), $\mathcal{C}_{j-1}\leftarrow \mathcal{C}_j\setminus \{\ell_j\}$, and $j\leftarrow j-1$	
    \ENDIF
    \STATE Sample $A_t\leftarrow \argmin_{k\in \mathcal{C}_{j}} N_k(t)$ (tie broken arbitrarily), 
		update $\{N_k(t)\}_{k\in \mathcal{C}_j} $ and $\hat{\bmu}(t)$
\ENDFOR
 \STATE $\ell_2\leftarrow \arg\min_{k\in \mathcal{C}_2}\hat{\mu}_k(T)$ (tie broken arbitrarily)
 \STATE Return $\hat{\imath}\leftarrow \arg\max_{k\in \mathcal{C}_2}\hat{\mu}_k(T)$ (tie broken arbitrarily)
\end{algorithmic}
\end{algorithm}

Consider the FB-BAI problem with $K$ arms as described in the introduction. For this problem, the SR algorithm starts by initializing the set of candidate arms as $\mathcal{C}_K=[K]$. The sampling budget is partitioned into $K-1$ phases. Following each phase, SR discards the empirically determined worst-performing arm from the candidate set. During each phase, SR adopts a uniform sampling strategy for the arms within the candidate set. 

The phase lengths are determined as follows. Define $\olog K\coloneqq1/2+\sum_{k=2}^K 1/k$. When the candidate set, denoted by $\mathcal{C}_j$, comprises more than two arms, i.e., $j> 2$, in the corresponding phase, SR works as follows: (i) each arm within $\mathcal{C}_j$ is sampled until the round $t$ at which $\min_{k\in \mathcal{C}_j}N_k(t)\ge T/(j\olog K)$. %
(ii) the arm identified as the empirical worst, denoted by $\ell_j$, is then discarded, which means $\mathcal{C}_{j-1}=\mathcal{C}_j\setminus \{\ell_j\}$. During the final phase, SR samples the two remaining arms evenly and recommends $\hat{\imath}$, the arm that exhibits the higher empirical mean in $\mathcal{C}_2$. Algorithm \ref{alg:SR} presents the pseudo-code of SR.

\subsection{Exact Analysis of SR}\label{sec:tight}
In Theorem 2 in \cite{wang2023best}, the authors show that SR satisfies for any $\bmu\in \Lambda$
\begin{equation}\label{eq:SR upper}
     \uplim_{T\rightarrow \infty} \frac{T}{\log (1/p_{\bmu ,T})}\le \max_{j=2,\ldots, K}\frac{j\olog K}{\Gamma_j(\bmu)}.
\end{equation}
where
\begin{equation*}%
\Gamma_j(\bmu)\coloneqq\min_{J\in \mathcal{J}_j(\bmu)}\inf_{\bl\in \Lambda:\lambda_{1(\bmu)}\le \min_{k\in J} \lambda_k}\sum_{k\in J}d(\lambda_k,\mu_k),
\end{equation*}
and $\mathcal{J}_j(\bmu)\coloneqq\{J\subseteq [K]:\left|J\right|=j,1(\bmu)\in J\}$. We show the bound (\ref{eq:SR upper}) is in fact tight. Indeed, in the following theorem whose proof is presented in Appendix~\ref{app:sr_K}, we derive a matching lower bound. 

\begin{theorem}\label{thm:tight_analysis_SR}
    Under the Successive Rejects algorithm \cite{audibert2010best}, for any $\bmu \in \Lambda$, 
    \begin{equation*}%
     \lim_{T\rightarrow \infty} \frac{T}{\log (1/p_{\bmu ,T})}= \max_{j=2,\ldots, K}\frac{j\olog K}{\Gamma_j(\bmu)}.
\end{equation*}
\end{theorem}

\subsection{Instances Where Uniform Sampling Outperforms SR}\label{sec:uni_and_SR}

We can use Theorem \ref{thm:tight_analysis_SR} to assess whether the SR algorithm is better than uniform. The next theorem shows that it is not even for three-armed bandits.

\begin{theorem}\label{thm:sr}
    There exists a three-armed bandit instance in which uniform sampling strictly outperforms SR asymptotically.
\end{theorem}
\begin{proof}
    From Theorem~\ref{thm:tight_analysis_SR}, for any $\bmu\in \Lambda$, the error rate of SR satisfies 
\begin{equation}\label{eq:sr_3}
    \lim_{T\rightarrow\infty}  \frac{T}{\log (1/p_{\bmu ,T})}=\max\left\{\frac{8}{3\Gamma_2(\bmu)},\frac{4}{\Gamma_3(\bmu)}\right\}.
\end{equation}
As for uniform sampling, the error rate satisfies \cite{glynn2004large} 
\begin{equation}\label{eq:uni_3}
     \lim_{T\rightarrow\infty}  \frac{T}{\log (1/p_{\bmu ,T})}=\frac{3}{\Gamma_2(\bmu)}.
\end{equation}
Thank to Proposition~\ref{prop:Gammaj} in Appendix~\ref{app:Gamma}, we can compute $\Gamma_2(0.5,0.3,0.3)\approx 0.0426387$  and $\Gamma_3(0.5,0.3,0.3)\approx 0.0562588$, which implies 
$$
(\ref{eq:sr_3})\approx71.1>70.3587\approx (\ref{eq:uni_3}) .
$$
We conclude that SR has a higher error rate than the uniform sampling algorithm in the instance $(0.5,0.3,0.3)$.
\end{proof}

We can use the results from Proposition~\ref{prop:Gammaj} in Appendix~\ref{app:Gamma} to numerically compare the error rates of the SR and uniform sampling algorithms. In Figure~\ref{fig:SR3}, we plot the set of instances $\bmu$ such $\mu_1>\mu_2=\mu_3$ where the SR algorithm has an higher error rate than the uniform sampling algorithm. 

\begin{figure}[h]
\centering
\includegraphics[width=0.34\textwidth]{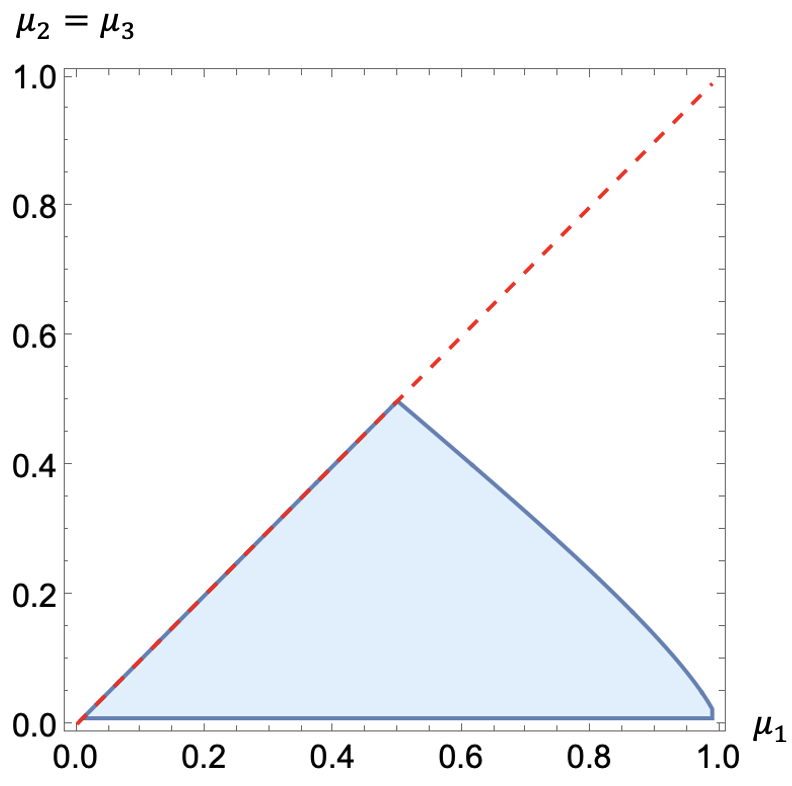}
\caption{The blue area corresponds to instances $\bmu=(\mu_1,\mu_2,\mu_2)$ such that $\mu_1>\mu_2=\mu_3$ and such that the uniform sampling algorithm strictly outperforms the SR algorithm. The red dashed line is the set of instances such that $\mu_1=\mu_2$.}\label{fig:SR3}
\end{figure}

\section{Related Work and Discussion}\label{sec:related}

In this section, we start by stating rigorously the two open problems in \cite{qin2022open} and that we address in this paper. We then present the related work for the FB-BAI problem in general, and finally discuss existing results for the two-arm case.   

\subsection{Open Problems Stated in \citet{qin2022open}}\label{subsec:related_open_qin}

{\bf Problem 1 \cite{qin2022open}} consists of investigating whether there exists a better than uniform algorithm that strictly outperforms uniform sampling on some instances. Based on Theorem~\ref{thm:main}, we can conclude that no such algorithms exist in two-armed bandits.

\medskip
\noindent
{\bf Problem 2 \cite{qin2022open}} consists in investigating whether the two following properties can hold simultaneously:
\begin{itemize}
\item[(a)] {\it Lower bound.} There exist an algorithm class $\mathcal{A}$ and a function $\Gamma^\star: \Lambda\mapsto \RR$ such that for any algorithm in $\mathcal{A}$,
\begin{equation*}%
    \forall \bmu\in \Lambda, \quad  \liminf_{T\rightarrow \infty} \frac{T}{\log (1/p_{\bmu,T})}\ge \Gamma^\star (\bmu).
\end{equation*}
\item[(b)]{\it Upper bound.} There is a single algorithm in $\mathcal{A}$ satisfies
\begin{equation*}%
    \forall \bmu\in \Lambda, \quad  \limsup_{T\rightarrow \infty} \frac{T}{\log (1/p_{\bmu,T})}\le \Gamma^\star (\bmu).
\end{equation*}
\end{itemize}
For problems with more than two arms, \cite{garivier2016optimal} conjecture that the lower bound discussed above for two-arm bandits can be generalized. However, again, \cite{ariu2021policy} prove that for any algorithm, when the number of arms is large, there exists at least one instance where the algorithm cannot reach the lower bound.
Our Theorem~\ref{thm:stable} addresses this open problem by considering ${\cal A}$ as the set of consistent and stable algorithms, by setting 
$$\Gamma^\star(\bmu)=\frac{1}{g(1/2,\bmu)},
$$
and by using the fact that the uniform sampling algorithm matches the corresponding lower bound.

\subsection{Fixed-Budget Best-Arm Identification}

The study of fixed-budget best-arm identification has been relatively recent \citep{audibert2010best, bubeck2011pure}, especially when compared to regret minimization \citep{lai1985asymptotically,cappe2013kullback} and fixed-confidence best-arm identification \citep{chernoff1959sequential,even2006action,garivier2016optimal}.  Since then, algorithms such as Successive Rejects (SR) \citep{audibert2010best}, Sequential Halving \citep{karnin2013almost}, and UGapE \citep{gabillon2012best} have been proposed with performance guarantees.

A first lower bound on the error rate for FB-BAI has been proposed in \cite{audibert2010best}. They prove that for any algorithm that knows the reward distributions of the arms, but does not know the order in which they correspond to the arms, there exists a bandit instance such that the probability of misidentification is lower bounded by $\exp(-\frac{c T}{H_2})$, where $c>0$ is some universal constant and $H_2 \coloneqq \max_{k\neq 1(\bmu)} \frac{k}{(\mu_{1(\bmu)}-\mu_k)^2}$.

In \cite{carpentier2016tight}, the authors prove that for any algorithm, there exist bandit instances such that the probability of error for one of the instances is lower bounded by $\exp(-\frac{400 T}{H_2 \log K})$, where $K$ is the number of the arms. The authors of \cite{ariu2021policy} revisit this lower bound, and demonstrate that no algorithm can universally achieve the same error probability as the best static algorithm, particularly when the number of arms is large. 

\cite{komiyama2022minimax} presents a minimax characterization of the error probability. They also conjecture that exploration that adapts to the instance is costly.  

The recent study \cite{wang2023best} investigates the FB-BAI problem using large deviation techniques. They relate the error probability to the Large Deviation Principle satisfied by the stochastic process capturing the empirical proportions of arm pulls and the sample means. 
Leveraging the connection, they not only enhance the guarantee of the Successive Rejects (SR) but also devise and analyze a novel algorithm with more adaptive rejections, Continuous Rejects (CR). The CR algorithm demonstrates superior performance both theoretically and numerically. Note however, that as SR, the CR algorithm is identical to the uniform sampling algorithm in bandit problems with two arms.

\subsection{A/B Testing}
For two-armed bandits, in \cite{kaufmann2014complexity,kaufmann2016complexity}, the authors try to provide a characterization
of the minimal instance-specific error probability for fixed-budget best-arm identification. For Bernoulli rewards, they establish a lower bound of this probability satisfied by any consistent algorithm:
\begin{align}\label{eq:lower_kaufmann}
   \forall \bmu\in \Lambda,\quad \liminf_{T\rightarrow \infty} \frac{T}{\log (1/p_{\bmu ,T})}  \ge \min_{x \in (0,1)}\frac{1}{g(x,\bmu)}.
\end{align}
\cite{kaufmann2014complexity,kaufmann2016complexity} also note, as in \cite{glynn2004large}, that the best static algorithm (that requires the knowledge of $\bmu$) matches this lower bound. They do not find an adaptive algorithm that universally matches the lower bound across all bandit instances. Our results state that this is indeed impossible.

Now for Gaussian rewards, \cite{kaufmann2014complexity,kaufmann2016complexity} derive an error probability lower bound satisfied by consistent algorithms. They also show that when the learner is aware of the variances of the rewards, one may find an algorithm whose performance matches this lower bound on all instances. This algorithm is a static algorithm that pulls the first arm $\sigma_1/(\sigma_1 + \sigma_2) T + o(T)$ times and the second arm $\sigma_2/(\sigma_1 + \sigma_2) T + o(T)$ times, where $\sigma_1^2$ and $\sigma_2^2$ are the variances of the rewards of arm 1 and 2, respectively. 

In \cite{degenne2023existence}, the author shows that, for Bernoulli bandits, a universally optimal algorithm matching the lower bound \eqref{eq:lower_kaufmann} does not exist.
Specifically, for any algorithm, there exists an instance and a static algorithm such that the considered algorithm performs strictly worse than the best static algorithm on the instance. 
Essentially, characterizing the instance-specific minimal error rate within a class of algorithms that includes all static algorithms is impossible. In this paper, we show that adaptive algorithms cannot even compete with a single static algorithm, namely the uniform sampling algorithm, on all instances.

\section{Conclusion}

In this paper, we investigated the problem of finding universally optimal algorithms for Fixed-Budget Best-Arm Identification (FB-BAI) in stochastic multi-armed bandits with Bernoulli rewards.
We found that, surprisingly, for two-armed bandits (the A/B testing problem), no algorithm strictly outperforms the uniform sampling algorithm.
We actually proved that within a natural and wide class of consistent and stable algorithms, uniform sampling is universally optimal.
Extending these results to the case of more than two arms is challenging. 
So far we have not found any adaptive algorithm outperforming uniform sampling for all instances.
For example, we were able to exactly characterize the asymptotic error probability of the celebrated Successive Rejects (SR) algorithm. 
As it turns out, SR is outperformed by uniform sampling in some instances.

Our study advances the understanding of the FB-BAI problem. 
However, we obtained a complete picture only for A/B testing with Bernoulli rewards. 
A similar picture is for now out of reach for general reward distributions. 
Indeed, the minimal error probability is not known even in the case of Gaussian reward distributions with unknown variances or in the case where these distributions are within one parameter exponential family. 
In the latter case, we conjecture that uniform sampling remains universally optimal. 
For problems involving more than two arms, the FB-BAI problem becomes even more challenging. We are currently working on extending the notion of stable algorithms, and on comparing their performance to that of the uniform sampling algorithm.

\section*{Impact Statement}
This paper focuses on discussing the existence of universally optimal algorithms for best-arm identification. 
It presents work whose goal is to advance the field of Machine Learning. There are many potential societal consequences of our work, none which we feel must be specifically highlighted here.

\section*{Acknowledgments}
We would like to express our gratitude to Chao Qin for meticulously examining the proofs in our early draft and identifying their incompleteness. His insightful comments led us to refine our definition of stable algorithms. 
Kaito Ariu's research is supported by JSPS KAKENHI Grant No. 23K19986. Alexandre Proutiere is supported by the Wallenberg AI, Autonomous Systems and Software Program (WASP) funded by the Knut and Alice Wallenberg Foundation, the Swedish Research Council (VR) and Digital Futures.

\bibliography{ref.bib,ref2.bib}
\bibliographystyle{icml2024}

\newpage
\appendix
\onecolumn
\section{Proof of Proposition \ref{prop:intersection}}\label{app:dual}
This section aims to prove Proposition \ref{prop:intersection}, restated below for convenience.

\begin{repproposition}{prop:intersection}
        For any $a\in (0,1),x\in (\frac{1}{2},1]$, there exists an instance $\bmu\in \Lambda$ such that \textnormal{(i)} $\mu_1>\mu_2,\,\mu_1+\mu_2\ge 1$, \textnormal{(ii)} $\lambda(x,\bmu)=a$, and \textnormal{(iii)} $x^\star(\bmu)< (\frac{1}{2}+x)/2$.
\end{repproposition}

\begin{proof}
Bernoulli distributions belong to the single-parameter exponential family. Thus, there is a strictly convex function $\phi:\RR\mapsto \RR$ specific to these distributions. We denote by $\bar{d}$ the corresponding Bregman divergence. More precisely, $\phi(\xi)]\coloneqq\log (1+e^\xi)$ with $\phi'(\xi)=\frac{e^\xi}{1+e^\xi }$ and $\phi'^{-1}(\mu)=\log \frac{\mu}{1-\mu}$. For a given $\bmu\in \Lambda$, there is a parameter $\bxi\in \RR^2$ with $\xi_1=\phi'^{-1}(\mu_1)$ and $\xi_2=\phi'^{-1}(\mu_2)$ (note that $\phi'$ is an invertible function). Let $\bar{\Lambda}=\{\bxi\in \RR^2:\xi_1\neq \xi_2\}=\phi'^{-1}(\Lambda)$ as the set of all parameters. 
$d(\mu_1,\mu_2)$ can be written as:
\begin{equation}\label{eq:dual_representation}
    d(\mu_1,\mu_2)=\bar{d}(\xi_2,\xi_1)=\phi(\xi_2)-\phi(\xi_1)-(\xi_2-\xi_1)\phi'(\xi_1).
\end{equation}
Following this formalism, we can present the conditions $(\overline{\textnormal{i}})$, $(\overline{\textnormal{ii}})$, and $(\overline{\textnormal{iii}})$ that are equivalent to the conditions (i), (ii) and (iii) used in the proposition. The proof and visualization for the lemma below can be found in Appendix~\ref{app:techlemmas} and Figure~\ref{fig:appA} respectively.

\begin{lemma}\label{lem:intersection dual} The statement of Proposition~\ref{prop:intersection} is equivalent to the following: 
    for any $\alpha \in \RR,\,x\in (\frac{1}{2},1]$, there exists an instance $\bxi\in \bar{\Lambda}$ such that $(\overline{\textnormal{i}})\,\xi_1>\xi_2,\xi_1\ge -\xi_2,\,(\overline{\textnormal{ii}})\, (1-x)\xi_1+x\xi_2=\alpha,\,$ and $(\overline{\textnormal{iii}})\, \bar{d}(\xi_1,(1-\tilde{x})\xi_1+\tilde{x}\xi_2) >\bar{d}(\xi_2,(1-\tilde{x})\xi_1+\tilde{x}\xi_2) $, where $\tilde{x}=(\frac{1}{2}+x)/2$.
\end{lemma}
 \begin{figure}[h]
\centering
\includegraphics[width=0.7\textwidth]{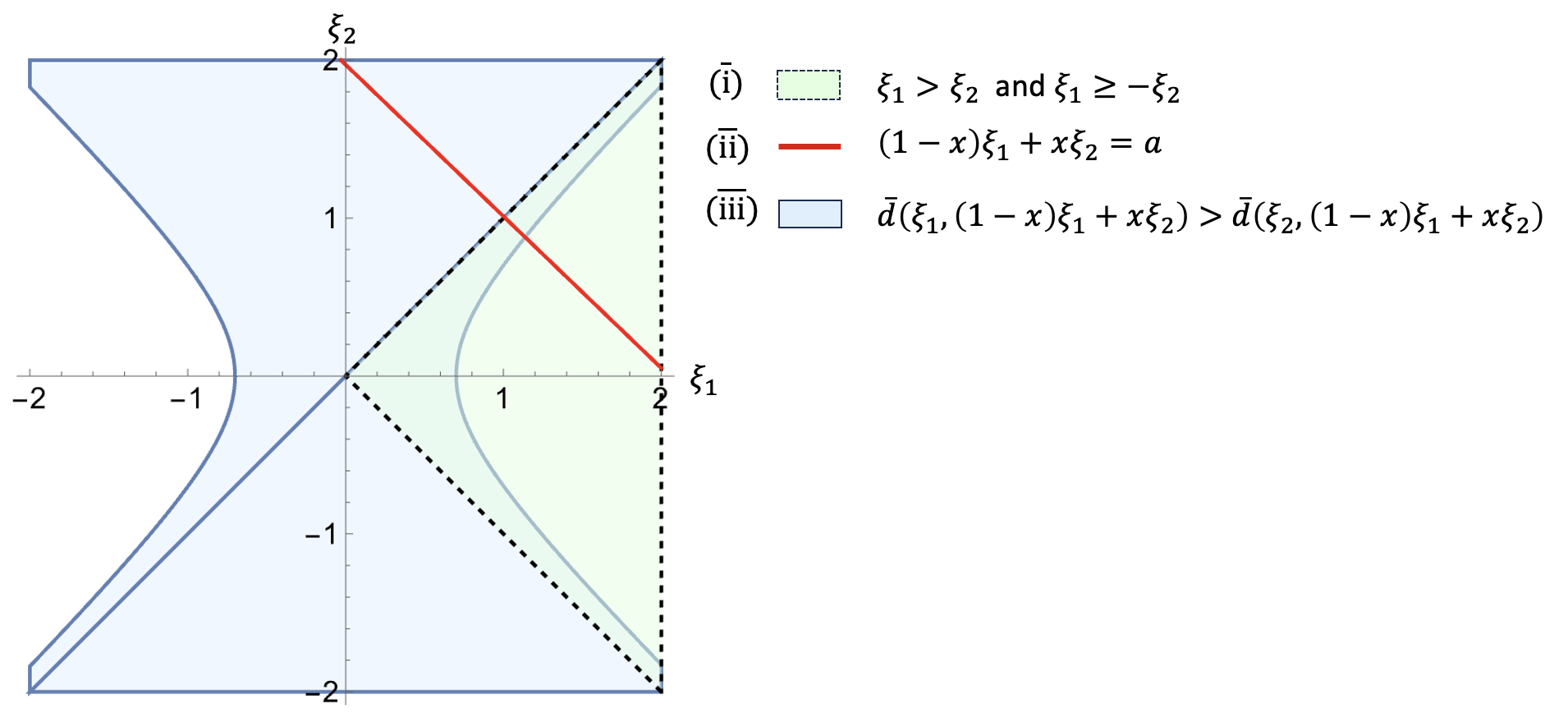}
\caption{
Visualization of Lemma~\ref{lem:intersection dual} with $a=0.55$ and $x=0.51$. The green region indicates $(\overline{\textnormal{i}})\,\xi_1>\xi_2,\xi_1\ge -\xi_2$. The red curve represents $(\overline{\textnormal{ii}})\, (1-x)\xi_1+x\xi_2=\alpha$. The blue region shows $(\overline{\textnormal{iii}})\, \bar{d}(\xi_1,(1-\tilde{x})\xi_1+\tilde{x}\xi_2) >\bar{d}(\xi_2,(1-\tilde{x})\xi_1+\tilde{x}\xi_2) $, where $\tilde{x}=(\frac{1}{2}+x)/2$.}\label{fig:appA}
\label{fig:my_label}
\end{figure}
\noindent
  We consider two cases: (a) when $\alpha < 0$ and (b) when $\alpha \ge 0$.
  
\medskip
  \noindent
  {\bf Case (a)} ($\alpha < 0$). We show that $\bxi=(-\alpha/(2x-1),\alpha/(2x-1))$ satisfies  $(\overline{\text{i}}),(\overline{\text{ii}}),(\overline{\text{iii}})$. As $\alpha<0, x>1/2$, $(\overline{\text{i}})$ follows directly. $(\overline{\text{ii}})$ follows as $(1-x)\xi_1+x\xi_2=(1-2x)\xi_1=\alpha$. As for $(\overline{\text{iii}})$, observe that 
    $$
    \tilde{\xi}=(1-\tilde{x})\xi_1+\tilde{x}\xi_2=  \frac{(2\tilde{x}-1)\alpha}{2x-1}<0.
    $$
    Hence $\phi'(\tilde{
    \xi })=\frac{e^{\tilde{\xi}}}{1+e^{\tilde{\xi}}}<\frac{1}{2}$. As a consequence,
\begin{align*}
    \bar{d}(\xi_1,\tilde{\xi})-\bar{d}(\xi_2,\tilde{\xi})&=\phi(\xi_1)-\phi(\xi_2)-\phi'(\tilde{\xi})(\xi_1-\xi_2)\\
    &=\log \left(\frac{1+e^{\xi_1}}{1+e^{-\xi_1}}\right)-\phi'(\tilde{\xi})2\xi_1\\
    &=\xi_1 \left(1-2\phi'(\tilde{\xi})\right)>0.
\end{align*}

\medskip
\noindent
{\bf Case (b)} ($\alpha \ge 0$). We claim that there is a half disk $\mathcal{D}:=\{\bxi\in \RR^2:\left\|\bxi-(\alpha,\alpha)\right\|_{\infty}<\delta,\,\xi_1>\xi_2\}$ for some $\delta>0$ such that $(\overline{\text{iii}})$ holds whenever $\bxi\in \mathcal{D}$. We then show the intersection of the disk and the line $\mathcal{L}:=\{\bxi\in \RR^2:(1-x)\xi_1+x\xi_2=\alpha\}$ is nonempty and satisfies $(\overline{\text{i}})$ and $(\overline{\text{ii}})$. 

Assume that $\xi_1>\xi_2$. On the one hand, applying Lemma \ref{lem:taylor} with $\alpha=\xi_1$ and $\beta=(1-\tilde{x})\xi_1+\tilde{x}\xi_2$ implies that there is $r_1\in [\xi_2,\xi_1]$ such that
\begin{equation}\label{eq:dbar1}
    \bar{d}(\xi_1,(1-\tilde{x})\xi_1+\tilde{x}\xi_2)=\frac{\phi''(r_1)\tilde{x}^2(\xi_1-\xi_2)^2}{2}.
\end{equation}
On the other hand, applying the same lemma with $\alpha = \xi_2$ and $ \beta = (1-\tilde{x})\xi_1+\tilde{x}\xi_2$, we find $r_2\in [\xi_2,\xi_1]$ such that
\begin{equation}\label{eq:dbar2}
    \bar{d}(\xi_2,(1-\tilde{x})\xi_1+\tilde{x}\xi_2)=\frac{\phi''(r_2)(1-\tilde{x})^2(\xi_1-\xi_2)^2}{2}.
\end{equation}
Since $\phi''(\xi)=\frac{e^\xi}{(1+e^\xi)^2}$ is a continuous function from $\RR$ to $\RR_{>0}$, Lemma \ref{lem:second derivative} implies that there exist $\delta>0$ such that $\min_{r\in [\xi_2,\xi_1]}\phi''(r) \tilde{x}^2>\max_{r\in [\xi_2,\xi_1]}\phi''(r) (1-\tilde{x})^2$ if $\left\|\bxi-(\alpha,\alpha)\right\|_{\infty}<\delta$. Thus by (\ref{eq:dbar1}) and (\ref{eq:dbar2}), we conclude that if $\bxi\in \mathcal{D}:=\{\bxi\in \RR^2:\left\|\bxi-(\alpha,\alpha)\right\|_{\infty}<\delta,\xi_1>\xi_2\}$, then $(\overline{\text{iii}})$ holds. Obviously, $\mathcal{D}\cap\mathcal{L}=\cup_{s\in (0,\delta/x)}\{(\alpha+xs,\alpha-(1-x)s) \}\neq \emptyset$. Consider $\bxi=(\alpha+xs,\alpha-(1-x)s)$ with some $s\in (0,\delta/x)$, we have $\xi_1>\alpha>\xi_2$ and $\xi_1+\xi_2=2\alpha+(2x-1)s\ge 0$, hence $\bxi$ satisfies~$(\overline{\text{i}})$. $(\overline{\text{ii}})$ holds directly by definition of ${\cal L}$.

\end{proof}

\subsection{Proof for the Closed-Form Expressions of $g(x,\bmu)$ and $\lambda(x,\bmu)$}
Here, we aim to present the proof of the following closed-form expressions for $g(x,\bmu)$ and $\lambda(x,\bmu)$:
\begin{lemma}\label{lem:closedform_glynn}
 For any $\bmu \in \Lambda$, for any $x\in (0,1)$, the following equations hold.
    \begin{align}
\lambda(x,\bmu) &=\frac{(\frac{\mu_1}{1-\mu_1})^{1-x}(\frac{\mu_2}{1-\mu_2})^x}{1+(\frac{\mu_1}{1-\mu_1})^{1-x}(\frac{\mu_2}{1-\mu_2})^x} ,\tag{\ref{eq:lambda}} \\
g(x,\bmu) & =-\log ((1-\mu_1)^{1-x}(1-\mu_2)^x+\mu_1^{1-x}\mu_2^x).\tag{\ref{eq:g alt}}
\end{align}
\end{lemma}
\begin{proof}
Again, we denote $\xi_1=\phi'^{-1}(\mu_1) = \log(\frac{\mu_1}{1 - \mu_1}),\,\xi_2=\phi'^{-1}(\mu_2)=\log(\frac{\mu_2}{1 - \mu_2})$ as in the proof of Proposition~\ref{prop:intersection}. The minimization problem 
\begin{equation}\label{eq:closeform-1}
    \min_{\lambda\in [0,1]} \big((1-x)d(\lambda,\mu_1)+xd(\lambda,\mu_2)\big)
\end{equation}
can be written as:
$$
 \min_{\lambda\in [0,1]}\big((1-x)\bar{d}(\xi_1,\phi'^{-1}(\lambda))+x\bar{d}(\xi_2,\phi'^{-1}(\lambda))\big),
$$
or equivalently,
\begin{equation}
\min_{\eta\in \RR}\big((1-x)\bar{d}(\xi_1,\eta) +x\bar{d}(\xi_2,\eta)\big).
\end{equation}
One can observe that 
$$
\frac{\partial}{\partial\eta} \big((1-x)\bar{d}(\xi_1,\eta) +x\bar{d}(\xi_2,\eta)\big)=\big(\eta-(1-x)\xi_1-x\xi_2\big)\phi''(\eta)
$$
has a unique root at the point $\eta=\eta(x,\bxi)=(1-x)\xi_1+x\xi_2$. As $\phi'$ is an invertible mapping, we conclude that $\phi'(\eta (x,\bxi))$ is the unique minimizer to the minimization problem (\ref{eq:closeform-1}). Therefore,
\begin{align*}
    \lambda(x,\bmu)&=\phi' ((1-x)\xi_1+x\xi_2)\\
    &=\phi'\left(\log\left(\left(\frac{\mu_1}{1-\mu_1}\right)^{1-x}\left(\frac{\mu_2}{1-\mu_2}\right)^x\right)\right)\\
    &=\frac{(\frac{\mu_1}{1-\mu_1})^{1-x}(\frac{\mu_2}{1-\mu_2})^x}{1+(\frac{\mu_1}{1-\mu_1})^{1-x}(\frac{\mu_2}{1-\mu_2})^x}.
\end{align*}
Finally, (\ref{eq:g alt}) can be obtained as follows.
\begin{align*}
    g(x,\bmu)& = (1-x)\bar{d}(\xi_1,\eta(x,\bxi))+x\bar{d}(\xi_2,\eta(x,\bxi))
    \\
    & = (1-x)\left(\phi(\xi_1) - \phi(\eta(x, \bxi)) - (\xi_1 - \eta(x, \bxi))\phi'(\eta(x, \bxi))\right) + x\left(\phi(\xi_2) - \phi(\eta(x, \bxi)) - (\xi_2 - \eta(x, \bxi))\phi'(\eta(x, \bxi))\right)
    \\
    & = (1-x)\phi(\xi_1) + x\phi(\xi_2) - \phi(\eta(x, \bxi)) -((1-x)\xi_1 + x\xi_2 - \eta(x, \bxi))\phi'(\eta(x,\bxi))
    \\
    & = (1-x)\phi(\xi_1) + x\phi(\xi_2) - \phi(\eta(x, \bxi))
    \\
    & = (1-x)\phi\left(\log(\frac{\mu_1}{1-\mu_1})\right) + x\phi\left(\log(\frac{\mu_2}{1-\mu_2})\right) - \phi\left((1-x)\log(\frac{\mu_1}{1-\mu_1}) + x\log(\frac{\mu_2}{1-\mu_2})\right)
    \\
    &= -\log ((1-\mu_1)^{1-x}(1-\mu_2)^x+\mu_1^{1-x}\mu_2^x).
\end{align*}
\end{proof}

\subsection{Proof for the Strong Concavity of $g(x,\bmu)$}

\begin{lemma}\label{lem:concave}
    For any $\bmu \in \Lambda$, for any $x\in (0,1)$,  $\frac{\partial^2}{\partial^2 x}g(x,\bmu)<0$.
\end{lemma}
\begin{proof}
    As shown in the proof of Lemma~\ref{lem:closedform_glynn}, 
    \begin{equation}\label{eq:concave-1}
         g(x,\bmu)=(1-x)\phi(\xi_1) + x\phi(\xi_2) - \phi(\eta(x, \bxi)),
    \end{equation}
    where $\xi_1=\phi'^{-1}(\mu_1),\,\xi_2=\phi'^{-1}(\mu_2)$, and $\eta(x,\bxi)=(1-x)\xi_1+x\xi_2$. %
    We differentiate (\ref{eq:concave-1}) with respect to $x$: 
    \begin{align*}
        \frac{\partial}{\partial x}g(x,\bmu)&=\phi(\xi_2)-\phi(\xi_1)+(\xi_1-\xi_2)\phi'(\eta(x,\bxi)),\\
        \frac{\partial^2}{\partial^2 x}g(x,\bmu)&=-(\xi_1-\xi_2)^2\phi''(\eta (x,\bxi))<0.
    \end{align*}
\end{proof}

\subsection{Proof of Proposition~\ref{prop:reflect}}\label{app:reflect}
\begin{repproposition}{prop:reflect}
  Denote $\bar{\bmu}=(1-\mu_2,1-\mu_1).$
    For any $x\in (0,1)$, for any $\bmu\in \Lambda$, $g(1-x,\bar{\bmu})=g(x,\bmu)$ and $\lambda (1-x,\bar{\bmu})=1-\lambda (x,\bmu)$.
\end{repproposition}
\begin{proof}
From \eqref{eq:g alt}, we obtain
\begin{align*}
    g(1-x,\bar{\bmu})   &=-\log ((1-\mu_1)^{1-x}(1-\mu_2)^x+\mu_1^{1-x}\mu_2^x)
    \\
    & =g(x,\bmu).
\end{align*}

Lastly, by \eqref{eq:lambda}, we have 
\begin{align*}
\lambda(1-x,\bar{\bmu})    &=\frac{(1-\mu_1)^{1-x}(1-\mu_2)^x}{(1-\mu_1)^{1-x}(1-\mu_2)^x+\mu_1^{1-x}\mu_2^{x} }\\
    &=1-\frac{\mu_1^{1-x}\mu_2^x}{(1-\mu_1)^{1-x}(1-\mu_2)^x+\mu_1^{1-x}\mu_2^{x} }
    \\
    & =1-\lambda(x,\bmu).
\end{align*}
This concludes the proof.

\end{proof}

\newpage
\section{Proof of Proposition \ref{prop:asymmetry}}\label{app:asym}

\begin{repproposition}{prop:asymmetry}
    Suppose $\mu_1>\mu_2$ and $\mu_1+\mu_2\ge 1$. For any positive $\delta\le \min \{x^\star(\bmu),1-x^\star(\bmu)\},$
    $$
    g(x^\star(\bmu)-\delta,\bmu)\ge g(x^\star(\bmu)+\delta,\bmu).
    $$
\end{repproposition}

\begin{proof}
For simplicity, let $x^\star=x^\star(\bmu)$, and define
\begin{equation}\label{eq:gbar}
\hat{g}(x,\bmu)=\exp(-g(x,\bmu))=(1-\mu_1)^{1-x}(1-\mu_2)^x+\mu_1^{1-x}\mu_2^x.
\end{equation}
Consider the function $f(\delta) := \hat{g}(x^\star + \delta, \bmu) - \hat{g}(x^\star - \delta, \bmu)$. In order to demonstrate that $f(\delta) \ge 0$ using the mean value theorem, we show that the first-order derivative, $f'(\delta)$, is non-negative for all $\delta \in (0, \min\{x^\star(\bmu), 1 - x^\star(\bmu)\})$.

\noindent
By definition of $x^\star$, we have $\frac{\partial \hat{g}(x^\star,\bmu)}{\partial x}=0$, i.e.,
$$
(1-\mu_1)^{1-x^\star}(1-\mu_2)^{x^\star} \log \left(\frac{1-\mu_2}{1-\mu_1}\right)+\mu_1^{1-x^\star}\mu_2^{x^\star}\log \left(\frac{\mu_2}{\mu_1}\right)=0.
$$
This implies that 
\begin{equation}\label{eq:M}
\frac{(1-\mu_1)^{1-x^\star}(1-\mu_2)^{x^\star}}{\log \left(\frac{\mu_1}{\mu_2}\right)}= \frac{\mu_1^{1-x^\star}\mu_2^{x^\star}}{\log\left(\frac{1-\mu_2}{1-\mu_1}\right) }.  
\end{equation}
Let $M$ be the value of (\ref{eq:M}). $M\ge 0$ as $\mu_1>\mu_2$. Recalling the definition of $f(\delta)$ and (\ref{eq:gbar}), we get
\begin{equation*}
f(\delta)=    M\log\left(\frac{\mu_1}{\mu_2}\right) \left[ \left(\frac{1-\mu_2}{1-\mu_1}\right)^\delta-\left(\frac{1-\mu_2}{1-\mu_1}\right)^{-\delta} \right]-M\log\left(\frac{1-\mu_2}{1-\mu_1}\right)\left[ \left(\frac{\mu_1}{\mu_2}\right)^\delta-\left( \frac{\mu_1}{\mu_2}\right)^{-\delta} \right]
\end{equation*}
and 
$$
f'(\delta)=   M\log\left(\frac{\mu_1}{\mu_2}\right)\log\left(\frac{1-\mu_2}{1-\mu_1}\right) \left[ \left(\frac{1-\mu_2}{1-\mu_1}\right)^\delta+\left(\frac{1-\mu_2}{1-\mu_1}\right)^{-\delta}-\left(\frac{\mu_1}{\mu_2}\right)^\delta-\left(\frac{\mu_1}{\mu_2}\right)^{-\delta}\right].
$$
Observe that the first three factors on the r.h.s. of the above expression are all positive, since $\mu_1 > \mu_2$. For the last factor, we first have $\frac{1 -\mu_2}{1- \mu_1}\ge \frac{\mu_1}{\mu_2}$, according to Lemma~\ref{lem:entropy} in Appendix~\ref{app:techlemmas}. Additionally, note that the mapping $z\mapsto z^\delta+z^{-\delta}$ is increasing when $z>0$. From these observations, we can conclude that $f'(\delta)\ge 0$, which completes the proof.
\end{proof}

\newpage

\section{Technical Lemmas}\label{app:techlemmas}

\begin{replemma}{lem:negation_stable}
   Let $a\in (0,1)$. Assume that the statement \textnormal{(B)} of Definition~\ref{def:stable} does not hold. Then for any $\{\bp^{(n)}\}_{n=1}^\infty\subset \{\bp\in \Lambda:\pi_1<\pi_2\}$ such that $\bp^{(n)}\xrightarrow{n\rightarrow\infty} (a,a)$, there exists a value $x \in [0, 1]$ and increasing subsequences of integers $\{n_{m}\}_{m=1}^\infty, \{T_{m, \ell}\}_{\ell=1}^\infty \subset \mathbb{N}$ such that 
    $$
\lim_{m \to \infty}\lim_{\ell \to \infty}  \EE_{\bp^{(n_m)}}[\omega_2(T_{m,\ell})] =x \neq \frac{1}{2}.
$$
\end{replemma}

\begin{proof}
      First, Lemma~\ref{lem:equivalent} shows that the equations
$$
    \lim_{n\rightarrow \infty} \liminf_{T\rightarrow \infty}\EE_{\bp^{(n)}}[\omega_2(T)]  =  
    \lim_{n\rightarrow \infty} \limsup_{T\rightarrow \infty}\EE_{\bp^{(n)}}[\omega_2(T)] =
    \frac{1}{2}
    $$
    in the statement \textnormal{(B)} of Definition~\ref{def:stable} is equivalent to the statement: 
    $\forall \varepsilon>0$, $\exists N\in \mathbb{N}$ such that $\forall n \ge N$, $\exists T_{n} \in \mathbb{N}$, $\forall T \ge T_{n}$,
    $$
\left| \EE_{\bp^{(n)}}[\omega_2(T)] - \frac{1}{2}\right|<\varepsilon.
    $$
Its negation is: there exists $\varepsilon \in (0, \frac{1}{2}]$ such that  
 $\forall N \in \mathbb{N}$, $\exists \bar{n} \ge N$ such that $\forall T \in \mathbb{N}$, $\exists \bar{T}_{\bar{n}} \ge T$ such that 
     \begin{equation}\label{eq:BW1}
         \left| \EE_{\bp^{(\bar{n})}}[\omega_2(\bar{T}_{\bar{n}})] - \frac{1}{2}\right|\ge\varepsilon.
     \end{equation}
    In the above statement, let select $N$ arbitrarily, and fix a corresponding $\bar{n}$. If we take $T=1$, we can find $\bar{T}_{\bar{n}}= \bar{T}_{\bar{n}, 1}\ge 1$ satisfying (\ref{eq:BW1}). Next, we take $T=\bar{T}_{\bar{n}, 1}+1$, we can find $\bar{T}_{\bar{n}}=\bar{T}_{\bar{n}, 2}> \bar{T}_{\bar{n}, 1}$ satisfying (\ref{eq:BW1}). By repeating this operation, we can construct an increasing sequence of integers $\{\bar{T}_{\bar{n},\ell}\}_{\ell=1}^\infty$ such that: 
     $$
\forall \ell \in \NN, \quad \left| \EE_{\bp^{(\bar{n})}}[\omega_2(\bar{T}_{\bar{n} ,\ell})] - \frac{1}{2}\right|\ge\varepsilon.
    $$
We have now proved that there exists $\varepsilon \in (0, \frac{1}{2}]$ such that  
 $\forall N \in \mathbb{N}$, $\exists \bar{n} \ge N$ such that  
    \begin{equation}\label{eq:BW2}
        \exists \{ \bar{T}_{\bar{n} ,\ell}\}_{\ell=1}^\infty : \ \ \forall \ell \in \NN, \bar{T}_{\bar{n} ,\ell}< \bar{T}_{\bar{n} ,\ell+1}\hbox{ and } \left| \EE_{\bp^{(\bar{n})}}[\omega_2(\bar{T}_{\bar{n} ,\ell})] - \frac{1}{2}\right|\ge\varepsilon.
    \end{equation}
Again, in the above statement, if we take $N=1$, we can find $\bar{n}_1 \ge 1$ satisfying (\ref{eq:BW2}).  
Next, if we take $N=\bar{n}_1+1$, we can find $\bar{n} =\bar{n}_2> \bar{n}_1$ satisfying (\ref{eq:BW2}). By repeating this operation, we can construct an increasing sequence of integers $\{\bar{n}_m\}_{m=1}^\infty$ satisfying (\ref{eq:BW2}). In summary, we have found an increasing sequence of integers $\{\bar{n}_m\}_{m=1}^\infty$ and for all $m$, an other increasing sequence of integers $\{ \bar{T}_{\bar{n}_m ,\ell}\}_{\ell=1}^\infty$ such that:
\begin{align}\label{eq:subsecs_dev}
\forall (m,\ell) \in \NN^2, \quad \left| \EE_{\bp^{(\bar{n}_m)}}[\omega_2(\bar{T}_{\bar{n}_m ,\ell})] - \frac{1}{2}\right|\ge\varepsilon.
\end{align}
From the Bolzano–Weierstrass theorem, a bounded sequence always contains a convergent subsequence. Thus, for each $m \in \mathbb{N}$, one can always find 
 $\{T_{\bar{n}_m,\ell}\}_{\ell=1}^\infty\subset \{\bar{T}_{\bar{n}_m,\ell}\}_{\ell=1}^\infty$ such that $\{ \EE_{\bp^{(\bar{n}_m)}}[\omega_2({T}_{\bar{n}_m,\ell})]\}_{\ell=1}^\infty$ converges. We denote by $x_{\bar{n}_m}$ its limit, i.e., $\lim_{\ell \to \infty} \EE_{\bp^{(\bar{n}_m)}}[\omega_2(T_{\bar{n}_m,\ell})] = x_{\bar{n}_m}$. Note that from \eqref{eq:subsecs_dev}, with some $\varepsilon \in (0, \frac{1}{2}]$,
\begin{align}\label{neq:subsec_converg}
\forall m \in \mathbb{N}, \quad \left|x_{\bar{n}_m} - \frac{1}{2}\right| \ge \varepsilon.
\end{align}
Futher observe that of course, $x_{\bar{n}_m}\in [0,1]$ for all $m$. Using the Bolzano–Weierstrass theorem again, there exists $ \{n_{m}\}_{m =1}^\infty \subset \{\bar{n}_{m}\}_{m=1}^\infty$ such that $x_{n_m}$ converges to $x \in [0,1]$, i.e, $\lim_{m\to\infty}x_{n_m} = x$. From \eqref{neq:subsec_converg}, we remark that $|x - \frac{1}{2}| \ge \varepsilon$. The constructed $x$, $\{n_m\}_{m=1}^\infty$, and $\{T_{m, \ell}\}_{\ell=1}^\infty$ satisfy the desired claim, which concludes the proof.
 
\end{proof}
\begin{lemma}\label{lem:equivalent}
Let $\{\psi(n,T)\}_{n,T= 1}^\infty$ be a double sequence of real numbers. The following two statements are equivalent:
\begin{align*}
(I) \qquad& \lim_{n\to \infty}\liminf_{T\to \infty} \psi(n,T)=\lim_{n\to \infty}\limsup_{T\to \infty} \psi(n,T)=\frac{1}{2}\\
(II) \qquad & \forall\varepsilon>0, \exists N\in \NN: \forall n\ge N, \left( \exists T_n\in \NN, \forall T\ge T_n,\ \ 
    \left|\psi(n,T)-\frac{1}{2}\right|<\varepsilon \right).
\end{align*}    
\end{lemma}
\begin{proof}
We first prove that {$(II)\Rightarrow (I)$.} We rewrite $(II)$ as follows:  $\forall\varepsilon>0$, $\exists N\in \NN$ such that $\forall n\ge N$, $\exists T_n\in \NN$, $\forall T\ge T_n$,
    \begin{align}\label{eq:doublelimit_low_up}
    \frac{1}{2} - \varepsilon< \psi(n,T) <\frac{1}{2} + \varepsilon.
    \end{align}
    By taking $\limsup_{T\to\infty}$ and  $\liminf_{T\to\infty}$ of \eqref{eq:doublelimit_low_up}, we obtain the following statement: $\forall\varepsilon>0$, $\exists N\in \NN$  such that $\forall n\ge N$, 
    $$
   \frac{1}{2}-\varepsilon \le \liminf_{T\to \infty} \psi(n,T)\le \limsup_{T\to \infty} \psi(n,T)\le \frac{1}{2}+\varepsilon. 
    $$
    Therefore, we get $(I)$.
     
    \medskip
    \noindent
   Next we prove that $(I)\Rightarrow (II)$. Observe that $\lim_{n\to \infty}\limsup_{T\to \infty} \psi(n,T)=1/2$ implies for any $\eta_1>0$, there exists $\overline{N}\in \NN$, such that $\forall n\ge \overline{N}$,  
\begin{equation}\label{neq:lem81}
         \limsup_{T\to \infty} \psi(n,T)\le \frac{1}{2}+\eta_1.
    \end{equation}
    As a consequence of (\ref{neq:lem81}), for any $\eta_1 >0, \exists \overline{N} \in \mathbb{N}$ such that $\forall n\ge \overline{N}$, for any $\eta_2>0$, $\exists \overline{T}_n\in \NN$ such that $\forall T\ge \overline{T}_n$,  
\begin{equation}\label{neq:lem82}
        \psi(n,T)\le \frac{1}{2}+\eta_1+\eta_2.
    \end{equation}
    Similarly, $\lim_{n\to \infty}\liminf_{T\to \infty} \psi(n,T)=1/2$ implies:  for any $\eta_1 >0, \exists \underline{N} \in \mathbb{N}$ such that $\forall n\ge \underline{N}$, for any $\eta_2>0$, $\exists \underline{T}_n\in \NN$ such that $\forall T\ge \underline{T}_n$,  
      \begin{equation}\label{neq:lem83}
        \psi(n,T)\ge \frac{1}{2}-\eta_1-\eta_2.
    \end{equation}
    Combing (\ref{neq:lem82}) and (\ref{neq:lem83}), for any $\eta_1>0$,   $\exists N(=\max\{\overline{N},\underline{N}\})$ such that $\forall n\ge N$, for any $\eta_2>0$, $\exists T_n(=\max\{\overline{T}_n,\underline{T}_n\})$ such that $\forall T\ge T_n$, $\frac{1}{2}-\eta_1-\eta_2\le \psi(n,T)\le \frac{1}{2}+\eta_1+\eta_2$. 
    By taking $\eta_1 = \varepsilon/2$ and $\eta_2 =\varepsilon/2$, we obtain the statement: for any $\varepsilon>0$, $\exists N\in \mathbb{N}$ such that $\forall n\ge N$, $\exists T_n\in \mathbb{N}$ such that $\forall T\ge T_n$, $\frac{1}{2}-\varepsilon\le \psi(n,T)\le \frac{1}{2}+ \varepsilon$, which completes the proof.

\end{proof}

\begin{replemma}{lem:intersection dual} The statement of Proposition~\ref{prop:intersection} is equivalent to the following: 
    for any $\alpha \in \RR,\,x\in (\frac{1}{2},1]$, there exists an instance $\bxi\in \bar{\Lambda}$ such that $(\overline{\textnormal{i}})\,\xi_1>\xi_2,\xi_1\ge -\xi_2,\,(\overline{\textnormal{ii}})\, (1-x)\xi_1+x\xi_2=\alpha,\,$ and $(\overline{\textnormal{iii}})\, \bar{d}(\xi_1,(1-\tilde{x})\xi_1+\tilde{x}\xi_2) >\bar{d}(\xi_2,(1-\tilde{x})\xi_1+\tilde{x}\xi_2) $, where $\tilde{x}=(\frac{1}{2}+x)/2$.
\end{replemma}

\begin{proof}
We show the equivalence of $({\text{i}})$ to $(\overline{\text{i}})$, $({\text{ii}})$ to $(\overline{\text{ii}})$, and $({\text{iii}})$ to $(\overline{\text{iii}})$ in the following. 

\medskip
\noindent
{\bf Equivalence of (i) to $(\overline{\text{i}})$.} Note that $\mu_1>\mu_2$ can be rewritten as $\phi'(\xi_1) > \phi '(\xi_2)$. As $\phi'$ is a strictly increasing function, it holds if and only if $\xi_1>\xi_2$. As for $\mu_1+\mu_2\ge 1$, its equivalent statement is
$$
1\le \phi'(\xi_1)+\phi'(\xi_2)=\frac{e^{\xi_1}+e^{\xi_2}+2e^{\xi_1+\xi_2}}{(1+e^{\xi_1})(1+e^{\xi_2})}.
$$
By rearranging the above inequality, we obtain $\xi_1>-\xi_2$.

\medskip
\noindent
{\bf Equivalence of (ii) to  $(\overline{\text{ii}})$.}
 Introduce the following notation.
\begin{align*}
    \bar{g}(x,\bxi)&\coloneqq\inf_{\bar{\lambda}\in\RR} (1-x)\bar{d}(\xi_1,\bar{\lambda})+x\bar{d}(\xi_2,\bar{\lambda})=g(x,\phi'(\xi_1),\phi'(\xi_2)),\\
    \bar{\lambda}(x,\bxi)&\coloneqq\argmin_{\bar{\lambda}\in \RR} (1-x)\bar{d}(\xi_1,\bar{\lambda})+x\bar{d}(\xi_2,\bar{\lambda})= \phi'^{-1}(\lambda (x,\phi'(\xi_1),\phi'(\xi_2))),\\
\hbox{and }\bar{x}^\star(\bxi)&\coloneqq\argmax_{x\in (0,1)}\bar{g}(x,\bxi)=x^{\star}(\phi'(\xi_1),\phi'(\xi_2) ). 
\end{align*}
We use the following lemma from \citet{degenne2023existence}.

\begin{lemma}[Lemma 19 in \cite{degenne2023existence}]\label{lem:deg}
For any $\bxi\in \bar{\Lambda},\,x\in [0,1]$, $\bar{\lambda}(x,\bxi)=(1-x)\xi_1+x\xi_2$ and $\bar{x}^\star(\bxi)=\frac{\xi_1-\eta(\bxi)}{\xi_1-\xi_2}$, where $\eta(\bxi)=\phi'^{-1}(\frac{\phi(\xi_1)-\phi(\xi_2)}{\xi_1-\xi_2})$.
\end{lemma}
\noindent
The equivalence of (ii) to  $(\overline{\text{ii}})$ directly follows from Lemma \ref{lem:deg}.

\medskip
\noindent
{\bf Equivalence of (iii) to $(\overline{\text{iii}})$.}
From Lemma \ref{lem:deg}, $ \bar{x}^\star(\bxi) = x^{\star}(\phi'(\xi_1),\phi'(\xi_2) ) < (\frac{1}{2} + x)/2 = \tilde{x}$ for a given $x\in (\frac{1}{2},1]$ is equivalent to 
$$
\phi'^{-1} \left(\frac{\phi(\xi_1)-\phi(\xi_2)}{\xi_1-\xi_2} \right)=\eta(\bxi)>(1-\tilde{x})\xi_1+\tilde{x}\xi_2.
$$
By denoting $\tilde{\xi}=(1-\tilde{x})\xi_1+\tilde{x}\xi_2$, we then arrange it as:
$$
\phi(\xi_1)-\phi(\tilde{\xi})-\phi'(\tilde{\xi})(\xi_1-\tilde{\xi})>\phi(\xi_2)-\phi(\tilde{\xi})-\phi'(\tilde{\xi})(\xi_2-\tilde{\xi}),
$$
that is, $\bar{d}(\xi_1,\tilde{\xi})>\bar{d}(\xi_2,\tilde{\xi})$. This concludes the proof of Lemma~\ref{lem:intersection dual}.

\end{proof}

\begin{lemma}\label{lem:taylor}
    Given $\alpha,\beta\in \RR$, there exists $\min\{\alpha,\beta\}\le r\le \max\{\alpha,\beta\}$ such that
    $$
    \bar{d}(\alpha,\beta)=\frac{(\alpha-\beta)^2\phi''(r)}{2}.
    $$
\end{lemma}
\begin{proof}
    The first and second derivatives of $\bar{d}$ are:
    $$
    \frac{\partial}{\partial \alpha}\bar{d}(\alpha,\beta)=\phi'(\alpha)-\phi'(\beta)\hbox{ and } \frac{\partial^2}{\partial^2 \alpha}\bar{d}(\alpha,\beta)=\phi''(\alpha).
    $$
    Using Taylor's equality, we have $r\in [\min\{\alpha,\beta\}, \max\{\alpha,\beta\}]$ such that
    $$
    \bar{d}(\alpha,\beta)=\bar{d}(\beta,\beta)+(\alpha-\beta) \frac{\partial}{\partial \alpha}\bar{d}(\beta,\beta) +\frac{(\alpha-\beta)^2}{2}\frac{\partial^2}{\partial^2 \alpha}\bar{d}(r,\beta) =\frac{(\alpha-\beta)^2\phi''(r)}{2}.
    $$
\end{proof}

\medskip
\noindent

\begin{lemma}\label{lem:second derivative}
Suppose $f:\RR\mapsto\RR_{>0}$ is a continuous function. For any $\alpha\in \RR, \,x\in (\frac{1}{2},1)$, there exist $\delta>0$ s.t. if  $\left\|\bxi-(\alpha,\alpha)\right\|_{\infty}<\delta$ and $\xi_1>\xi_2$, then
    $$
    \min_{r\in [\xi_2,\xi_1]}f(r) x^2>\max_{r\in [\xi_2,\xi_1]}f(r) (1-x)^2.
    $$
\end{lemma}
\begin{proof}
As $x\in (\frac{1}{2},1)$, we derive $\frac{1}{4x^2}<1<\frac{1}{4(1-x)^2}$. By the continuity of $f$ and $f(\alpha)>0$, there exists $\delta>0$ such that if $\left|r-\alpha\right|<\delta$,
$$
\frac{f(\alpha)}{4x^2}< f(r)< \frac{f(\alpha)}{4(1-x)^2}.
$$
Consequently, when $\left\|\bxi-(\alpha,\alpha)\right\|_{\infty}<\delta$ and $\xi_1>\xi_2$, 
$$
\min_{r\in [\xi_2,\xi_1]}f(r)\ge \min_{\left|r-\alpha\right|<\delta}f(r)>\frac{f(\alpha)}{4x^2}\hbox{ and }\max_{r\in [\xi_2,\xi_1]}f(r)\le \max_{\left|r-\alpha\right|<\delta}f(r)<\frac{f(\alpha)}{4(1-x)^2},
$$
which yields that $\min_{r\in [\xi_2,\xi_1]}f(r)x^2>\frac{f(\alpha)}{4}>\max_{r\in [\xi_2,\xi_1]}f(r)(1-x)^2$.
\end{proof}

\medskip

\begin{lemma}\label{lem:entropy}
    If $\mu_1>\mu_2$ and $\mu_1+\mu_2\ge 1$, then $\frac{1-\mu_2}{1-\mu_1}\ge \frac{\mu_1}{\mu_2}$.
\end{lemma}
\begin{proof}
Observe that the given two assumptions can be rewritten as $\mu_1-\frac{1}{2}>\mu_2-\frac{1}{2}$ and $\frac{1}{2}-\mu_2\le \mu_1-\frac{1}{2}$. Thus, we conclude that
$$
(1-\mu_2)\mu_2=- \left(\mu_2-\frac{1}{2}\right)^2+\frac{1}{4}\ge -\left(\mu_1-\frac{1}{2}\right)^2+\frac{1}{4}=(1-\mu_1)\mu_1,
$$
which is equivalent to the desired conclusion.
\end{proof}

\newpage
\section{Proof of Theorem~\ref{thm:tight_analysis_SR}}\label{app:sr_K}
Throughout this section, we assume, without loss of generality, that $\mu_1>\mu_2\ge \ldots\ge \mu_K$.  
Recall that
\begin{equation}\label{eq:Gamma_appD}
\Gamma_j(\bmu)=\min_{J\in \mathcal{J}_j(\bmu)}\inf_{\bl\in \Lambda:\lambda_{1}\le \min_{k\in J} \lambda_k}\sum_{k\in J}d(\lambda_k,\mu_k),
\end{equation}
and $\mathcal{J}_j(\bmu)=\{J\subseteq [K]:\left|J\right|=j,1\in J\}$.
\begin{proof}[Proof of Theorem~\ref{thm:tight_analysis_SR}]
As mentioned in Section~\ref{sec:tight}, \cite{wang2023best} establishes that
\begin{equation*}\tag{\ref{eq:SR upper}}
     \limsup_{T\rightarrow \infty} \frac{T}{\log (1/p_{\bmu ,T})}\le \max_{j=2,\ldots, K}\frac{j\olog K}{\Gamma_j(\bmu)},\,\forall \bmu\in \Lambda.
\end{equation*}
Hence we just need to prove the following Lemma~\ref{lem:SR_lower}.
\begin{lemma}\label{lem:SR_lower}
 Under Algorithm~\ref{alg:SR}, one has     \begin{equation}\label{eq:SR lower}
    \liminf_{T\rightarrow \infty} \frac{T}{\log (1/p_{\bmu ,T})}\ge \max_{j=2,\ldots, K}\frac{j\olog K}{\Gamma_j(\bmu)},\,\forall \bmu\in \Lambda.
\end{equation}
\end{lemma}
Combining (\ref{eq:SR upper}) and (\ref{eq:SR lower}), we then complete the proof.
\end{proof}

\begin{proof}[Proof of Lemma~\ref{lem:SR_lower}]

For all $j=2,\ldots, K$, let ${\cal E}_j=\{\ell_j=1\}$ be the event that SR discards the best arm, $1$, when there are $j$ candidates remaining. As under the event $\cup_{j=2}^K{\cal E}_j$, $1$ is removed before the end, we have 
$$
p_{\bmu,T}=\sum_{j=2}^K\PP_{\bmu}[{\cal E}_j]\ge \PP_{\bmu}[{\cal E}_j],\,\quad\forall j=2,\ldots, K.
$$
Consequently, the above inequalities imply that 
$$
\liminf_{T\to\infty}\frac{T}{\log (1/p_{\bmu,T})}\ge \max_{j=2,\ldots, K}\liminf_{T\to\infty}\frac{T}{\log (1/\PP_{\bmu}[{\cal E}_j])}.
$$
The proof of the lemma is reduced to showing that for any $j=2,\ldots, K$,
\begin{equation}\label{eq:SRj lower}
\liminf_{T\to\infty}\frac{T}{\log (1/\PP_{\bmu}[{\cal E}_j])}\ge \frac{j\olog K}{\Gamma_j(\bmu)}.
\end{equation}

\noindent
{\bf Proof of (\ref{eq:SRj lower}).} Let $j\in \{2,\ldots, K\}$.
For convenience, we set $\mu_{K+1}$ equal to $0$. We then consider some $\bl\in \Lambda_j$, where
$$
\Lambda_j=\left\{\bl\in (0,1)^K\times  \{0\}:\mu_{j+1}<\lambda_1\le \min_{k\in [j]}\lambda_k,\text{ and }\lambda_k=\mu_k,\forall k\ge j+1\right\}.
$$
Applying a standard change-of-measure argument, as in the proof of Theorem~\ref{thm:main}, yields that 
\begin{equation}\label{eq:SRj lower-1}
    \sum_{k=1}^K\EE_{\bmu}[N_k(\lfloor \theta T\rfloor)]d(\lambda_k,\mu_k)\ge d(\PP_{\bl}[{\cal E}_j],\PP_{\bmu}[{\cal E}_j]) \ge\PP_{\bl}[{\cal E}_j]\log\left(\frac{1}{\PP_{\bmu}[{\cal E}_j]}\right) -\log 2,
\end{equation}
where the last inequality stems from the fact that $d(p,q)\ge p\log(1/q)-\log 2$ for all $p,q\in [0,1]$. Thanks to law of large number, as $T\to \infty$, $\EE_{\bmu}[\omega_k(\lfloor\theta T\rfloor)]\to\frac{1}{j\olog K}$ for each $k=1,\ldots,j$, and $\PP_{\bl}[{\cal E}_j]\to 1$. Taking $T$ to infinity in (\ref{eq:SRj lower-1}) yields that 
$$
\liminf_{T\to\infty}\frac{T}{\log (1/\PP_{\bmu}[{\cal E}_j])}\ge \frac{j\olog K}{\sum_{k=1}^jd(\lambda_k,\mu_k)}.
$$
Since the above inequality holds for all $\bl\in \Lambda_j$, we have 
\begin{align*}
\liminf_{T\to\infty}\frac{T}{\log (1/\PP_{\bmu}[{\cal E}_j])}&\ge\sup_{\bl\in \Lambda_j} \frac{j\olog K}{\sum_{k=1}^jd(\lambda_k,\mu_k)}\\
&=\frac{j\olog K}{\inf_{\bl \in \Lambda_j}\sum_{k=1}^jd(\lambda_k,\mu_k)}=\frac{j\olog K}{\Gamma_j(\bmu)},
\end{align*}
where the last equation directly follows from (\ref{eq:Gammaj-2}) in Proposition~\ref{prop:Gammaj} in Appendix~\ref{app:Gamma}.
    
\end{proof}

\subsection{Derivation of Computationally Tractable Form of $\Gamma_j(\bmu)$}\label{app:Gamma}
Throughout this subsection, we define the function $\Psi_j:\RR^j\mapsto\RR$ as:
\begin{equation}\label{eq:Psi_j}
    \Psi_j(x_1,\ldots,x_j)\coloneqq \inf_{\bet\in \RR^j:\eta_{1}\le \min_{k\in [j]} \eta_k}\sum_{k=1}^j\bar{d}(x_k,\eta_k).
\end{equation}
In the following, we establish important properties of $\Psi_j$ for a fixed $j\in \{2,\ldots, K\}$. These properties will help us to understand $\Gamma_j(\bmu)$.

\medskip
\begin{proposition}\label{prop:Psi_j-1}
If $x_1>x_2\ge\ldots\ge x_j$, then
\begin{equation}\label{eq:Psi_j-1}
     \Psi_j(x_1,\ldots,x_j)=\frac{\sum_{k\in \Ij (\xx)}\bar{d}(x_k, \frac{\sum_{k'\in \Ij (\xx)} x_{k'}}{\left|\Ij (\xx)\right|})}{\left|\Ij(\xx)\right|},
\end{equation}
where 
$\Ij(\xx)=\left\{i \in \{2,\ldots, j\} :  x_i(j-i+1)<x_1+\sum_{i< k\le j} x_k\right\}\cup \{1\}$.
Moreover, the minimizer $\bet^\star$ of \eqref{eq:Psi_j} satisfies: $x_j<\eta_1^\star\le \min_{k\in [j]}\eta_k^\star$, i.e., 
\begin{equation}\label{eq:Psi_j-2} \Psi_j(x_1,\ldots,x_j)=\inf_{\bet\in \RR^j:x_j<\eta_1\le \min_{k\in [j]}\eta_k} \sum_{k=1}^j \bar{d}(x_k,\eta_k).
\end{equation}
\end{proposition}
\begin{proof}
       Observe that in the optimization problem (\ref{eq:Psi_j}), there exists $\bet\in \RR^j$ such that all the constraints are strict (satisfying Slater's condition). Thus, the solution of (\ref{eq:Psi_j}) can be identified by verifying the KKT conditions. The corresponding Lagrangian function is 
    $$
    {\cal L}(\bet,\alpha_2,\ldots,\alpha_j)=\sum_{k=1}^j\bar{d}(x_k,\eta_k)+\sum_{k=2}^j\alpha_k(\eta_1-\eta_k),\,\forall (\bet,\alpha_2,\ldots,\alpha_j)\in \RR^j\times \RR_{\ge 0}^{j-1}.
    $$
    Let $(\bet^\star, \alpha_2^\star,\ldots,\alpha_j^\star)$ be a saddle point of $\mathcal{L}$. It satisfies KKT conditions:
    \begin{align}
\eta_1^\star\le\eta^\star_k,\forall k=2,\ldots ,j,\tag{Primal Feasibility}\\
\alpha_k^\star\ge 0,\,\forall k=2,\ldots,j,\tag{Dual Feasibility}\\
\frac{\partial}{\partial\eta_k}{\cal L}(\bet^\star,\alpha_2^\star,\ldots,\alpha_j^\star )=0, \,\forall k=1,\ldots,j,\tag{Stationarity}\\
\alpha^\star_k(\eta_1^\star-\eta_k^\star)=0, \,\forall k=2,\ldots,j.\tag{Complementarity}
    \end{align}
    Recall that $\bar{d}(x,\eta)=\phi(x)-\phi(\eta)-(x-\eta)\phi'(\eta)$ (see (\ref{eq:dual_representation})), the partial differentiation on the second argument is hence $  \frac{\partial }{\partial \eta}\bar{d}(x,\eta)=\phi''(\eta)(\eta-x)$. We rewrite the above stationarity condition as:
    \begin{equation}
        \phi''(\eta_1^\star)(\eta_1^\star-x_1)+\sum_{k=2}^j\alpha_k^\star  =0;\,\phi''(\eta_k^\star)(\eta_k^\star-x_k)=\alpha_k^\star ,\,\forall k=2,\ldots,j,\tag{Stationarity}
    \end{equation}
Next observe that $\forall i\in \Ij(x)\setminus\{1\}$, it holds that 
    \begin{align}\label{eq:D2-1}
     \nonumber   \frac{\sum_{k\in\Ij(\xx)}x_k}{\left|\Ij(\xx)\right|}-x_i&=\frac{1}{\left|\Ij(\xx)\right|}\left(\sum_{k\in\Ij(\xx)}x_k-\left|\Ij(\xx)\right|x_i\right)\\
        &\ge \frac{1}{\left|\Ij(\xx)\right|}\left(x_1+\sum_{i<k\le j}x_k-(j-i+1)x_i\right)>0,
    \end{align}
    where the first inequality stems from $x_1> x_2\ge \ldots\ge x_K$, and the last one holds directly from the definition of $\Ij(\xx)$. 
    
     One can verify $(\bet^\star,\alpha^\star_2,\ldots,\alpha^\star_j)$ defined below satisfies the KKT conditions listed above.
    $$
    \eta^\star_i=\left\{
    \begin{array}{ll}
      \frac{\sum_{k\in\Ij(\xx)}x_{k}}{\left|\Ij(\xx)\right|} ,  &\text{if }i\in \Ij(\xx),  \\
        x_i, & \text{otherwise,}
    \end{array}
    \right.
     \alpha^\star_i=\left\{
    \begin{array}{ll}
      \phi''( \frac{\sum_{k\in \Ij(\xx)}x_{k}}{\left|\Ij(\xx)\right|})\left(\frac{\sum_{k\in\Ij(\xx)}x_{k}}{\left|\Ij(\xx)\right|}-x_i\right) ,  &\text{if }i\in \Ij(\xx),  \\
        0, & \text{otherwise,}
    \end{array}
    \right.
    $$
 where $\alpha^*_i\ge 0,\,\forall i\in \Ij(\xx)$ as (\ref{eq:D2-1}).
 As for (\ref{eq:Psi_j-2}), we observe that $\{1,j\}\in \Ij(\xx)$ as $x_1>x_j$. Hence the above minimizer $\bet^\star$ needs to satisfy that  $x_j<\eta_1^\star\le \min_{k\in [j]}\eta_k^\star$, and (\ref{eq:Psi_j-2}) follows directly.
\end{proof}

\medskip
\begin{proposition}\label{prop:Psi_j-2}
    Let $(x_1,\ldots,x_j) \in \RR^j$ be such that $x_1 > x_2 \ge \ldots \ge x_j$. Then, for any $k \neq 1$, we have $\frac{\partial}{\partial x_k}\Psi_j (x_1,\ldots,x_j) \le 0$. Consequently, if $\yy,\yy’ \in \{\zz \in \RR^j: z_1 > z_2 \ge \ldots \ge z_j\}$ are such that $y_1 = y’_1$ and $y_k \ge y’_k$ for all $k = 2,\ldots,j$, then it follows that $\Psi_j(y_1,y_2,\ldots,y_j) \le \Psi_j(y_1,y’_2,\ldots,y’_j)$.
\end{proposition}
\begin{proof}
Recall that $\bar{d}(x,\eta)=\phi(x)-\phi(\eta)-(x-\eta)\phi'(\eta)$ (see (\ref{eq:dual_representation})), together with (\ref{eq:Psi_j-1}) in Proposition~\ref{prop:Psi_j-1}, one can deduce that
\begin{equation}\label{eq:Psi_j-3}
\Psi_j(\xx)=\sum_{k\in \Ij(\xx)}\phi(x_k)-\left|\Ij(\xx)\right|\phi(\frac{\sum_{k\in \Ij(\xx)}x_k}{\left|\Ij(\xx)\right|}).
\end{equation}
From the r.h.s of (\ref{eq:Psi_j-3}), the partial differential on the $k$-th coordinate yields that
$$
\frac{\partial}{\partial x_k}\Psi_j (x_1,\ldots,x_j)=
\left\{\begin{array}{ll}
   0 , & \hbox{if }k\notin \Ij(\xx), \\
   \phi'(x_k)-\phi'(\frac{\sum_{k\in \Ij(\xx)}x_k}{\left|\Ij(\xx)\right|}),  & \hbox{otherwise. } 
\end{array}\right.
$$
Let $i\in \Ij(\xx)\setminus \{1\}$, the definition of $\Ij(\xx)$ implies that 
$$
x_i<\frac{\sum_{i<k\le j}x_k}{j-i+1}\le \frac{\sum_{k\in \Ij(\xx)}x_k}{\left|\Ij(\xx)\right|}.
$$
Because $\phi'$ is a strictly increasing function, we then conclude that  $    \frac{\partial}{\partial x_k}\Psi_j (x_1,\ldots,x_j)\le 0,\,\forall k=2,\ldots, j$.
\end{proof}

\medskip
\noindent
We next present the useful forms for $\Gamma_j(\bmu)$ for any $j=2,\ldots, K$.

\medskip

\begin{proposition}\label{prop:Gammaj}
   Let $\xi_k=\phi’^{-1}(\mu_k)=\log\frac{\mu_k}{1-\mu_k},\,\forall k\in [K]$ with $\mu_1>\mu_2\ge \ldots\ge \mu_K$, where $\phi$ is the strictly convex function shown in Appendix~\ref{app:dual}.  Then
   \begin{equation}\label{eq:Gammaj-1}
        \Gamma_j(\bmu)=\frac{\sum_{k\in \Ij (\bmu)}\bar{d}(\xi_k, \frac{\sum_{k'\in \Ij (\bmu)} \xi_{k'}}{\left|\Ij (\bmu)\right|})}{\left|\Ij(\bmu)\right|},
   \end{equation}
where $\Ij(\bmu)=\left\{i \in \{2,\ldots, j\} :  \xi_i(j-i+1)<\xi_1+\sum_{i< k\le j}\xi_k\right\}\cup \{1\}$. Moreover, the minimizer $(J^\star, \bl^\star)$ of \eqref{eq:Gamma_appD} satisfies, $J^\star =[j]$, $\mu_{j+1}<\lambda_1^\star \le \min_{k\in [j]}\lambda_k^\star$, and $\lambda_k^\star=\mu_k\;\forall k\ge j+1$, i.e., 
\begin{equation}\label{eq:Gammaj-2}
     \Gamma_j(\bmu)=\inf_{\bl\in \Lambda_j}\sum_{k=1}^jd(\lambda_k,\mu_k),
\end{equation}
where
$$
\Lambda_j=\left\{\bl\in (0,1)^K:\mu_{j+1}<\lambda_1\le \min_{k\in [j]}\lambda_k,\text{ and }\lambda_k=\mu_k,\forall k\ge j+1\right\}.
$$
\end{proposition}
\begin{proof}
    Recall that $\Gamma_j(\bmu)=\min_{J\in \mathcal{J}_j(\bmu)}\inf_{\bl\in \Lambda:\lambda_1\le \min_{k\in J} \lambda_k}\sum_{k\in J}d(\lambda_k,\mu_k)$. The fact that $d(\lambda_k,\mu_k)=\bar{d}(\xi_k,\phi'^{-1}(\lambda_k))$ (see (\ref{eq:dual_representation})) implies that 
    \begin{align}\label{eq:Gammaj-3}
\Gamma_j(\bmu)&=\min_{J\in \mathcal{J}_j(\bmu)}\inf_{\bet\in \RR^j:\lambda_{1}\le \min_{k\in J} \lambda_k}\sum_{k\in J}\bar{d}(\xi_k,\eta_k)\nonumber\\
&=\min_{J\in {\cal J}_j(\bmu)}\Psi_j(\xi_{J_1},\ldots,\xi_{J_j}),
\end{align}
where $J_k$ denotes the $k$-th smallest index in $J$. Recall that $\mathcal{J}_j(\bmu)=\{J\subseteq [K]:\left|J\right|=j,1(\bmu)\in J\}$, hence $ [j] \in \mathcal{J}_j(\bmu)$.
Therefore, from Proposition~\ref{prop:Psi_j-2} and $\xi_1>\xi_2\ge \ldots\ge \xi_K$, we deduce that
\begin{equation}\label{eq:Gammaj-4}
    \min_{J\in {\cal J}_j(\bmu)}\Psi_j(\xi_{J_1},\ldots,\xi_{J_j})=\Psi_j(\xi_1,\ldots,\xi_j).
\end{equation} 
(\ref{eq:Gammaj-1}) follows as the consequence of (\ref{eq:Gammaj-3}), (\ref{eq:Gammaj-4}), and (\ref{eq:Psi_j-1}) in Proposition~\ref{prop:Psi_j-1}. \\
\noindent
As for (\ref{eq:Gammaj-2}), (\ref{eq:Gammaj-3}), (\ref{eq:Gammaj-4}), and (\ref{eq:Psi_j-2}) in Proposition~\ref{prop:Psi_j-1} yield that 
\begin{equation}\label{eq:Gammaj-5}
    \Gamma_j(\bmu)=\inf_{\bet\in \RR^j:\xi_j<\eta_1\le \min_{k\in [j]}\eta_k} \sum_{k=1}^j \bar{d}(\xi_k,\eta_k).
\end{equation}
Using fact that $\bar{d}(\xi_k,\phi'^{-1}(\lambda_k))=d(\lambda_k,\mu_k)$ again, one can derive (\ref{eq:Gammaj-2}) from (\ref{eq:Gammaj-5}).
\end{proof}

\newpage
\section{Examples of Stable Algorithms}\label{app:eftt}

In this section, we present various examples of stable algorithms (Definition~\ref{def:stable}). We show that algorithms following one of the design principles below are stable. We assume that in all cases, there is an initialization phase where each arm is sampled $\lfloor \alpha T\rfloor$ times for some $\alpha>0$. This ensures that the arm rewards will be estimated accurately and that the algorithms are consistent. In the second phase, the algorithm design can be:
\begin{enumerate}
    \item {\it Uniform Sampling if Empirically Close.} The algorithm equally samples arms whenever the estimated gap $| \hat{\mu}_1(\tau)-\hat{\mu}_2(\tau)|$ of the mean arm rewards on the $\tau=\lfloor\alpha T\rfloor$-th round falls below a fixed threshold $\varepsilon>0$ is stable. No rules are added if the estimated gap is above the threshold. The algorithm could, for example, use the estimated optimal static exploration rate $x^\star(\hat\bmu(t))=\argmax_x g(x,\hat\bmu(t))$. The algorithm with such a choice is referred to as ETT (Estimate and Thresholded Tracking), and it is discussed in \ref{app:ETT}.
    \item {\it Track a Symmetric Continuous Function of the Empirical Rewards.} Here, the algorithm samples arms so that up to round the $t$-th round, arm 2 has been sampled $tf(\hat{\bmu}(t))$ where $f$ is a continuous function satisfying $f(a,a)=1/2$ for any $a$ and $\hat{\bmu}(t)$ denotes the empirical rewards at round $t$. We refer to this kind of algorithm as TCSF (Track-a-Continuous-Symmetric-Function), and it is discussed in \ref{app:TCSF}. 
\end{enumerate}

We present these algorithms in detail below and establish their stability. We note that the class of algorithms satisfying one of the above design principles is wide, and this makes the class of stable and consistent algorithms relevant. Simple numerical experiments are presented at the end of this section, in \ref{app:num}.

\subsection{The ETT Algorithm}\label{app:ETT}

The pseudo-code of ETT is presented in  Algorithm~\ref{alg:EFTT}.

\begin{lemma}\label{lem:EFTT_is_stable}
The algorithm ETT with input $\alpha\in (0,1/2)$ and $\varepsilon>0$ is stable.
\end{lemma}

\begin{algorithm}
\caption{Estimate and Thresholded Tracking (ETT)}\label{alg:EFTT}
\begin{algorithmic}[1]
\STATE {\bf Input:} $\alpha>0$, $\varepsilon>0$
\STATE Play each arm $\max\{\lfloor \alpha T\rfloor,1\}$ times
\STATE $\tau\leftarrow 2\max\{\lfloor \alpha T\rfloor,1\}$
\STATE Estimate the optimal allocation $\hat{x}^\star\leftarrow \argmax_x g(x,\hat{\bmu}(\tau))$
\IF {$\left|\hat{\mu}_1(\tau)-\hat{\mu}_2(\tau)\right|>\varepsilon$}
\FOR {$t=\tau+1,\ldots, T$}
    
        \STATE $\text{play } A_t \leftarrow\left\{\begin{array}{cc}
                2 &\text{ if }\hat{x}^\star>\frac{N_2(t)}{t},  \\
                1& \text{otherwise}
              \end{array}\right.$
\ENDFOR
\ELSE
\FOR {$t=\tau+1,\ldots, T$}
        \STATE play $A_t\leftarrow\argmin_{k}N_k(t)$ (tie broken arbitrarily)
        \ENDFOR

    \ENDIF

\STATE $\hat{\imath} \gets \argmax_{k \in \{1, 2\}} \hat{\mu}_k(T)$ (tie broken arbitrarily)
\STATE {\bf Output:}  $\hat{\imath}$
\end{algorithmic}
\end{algorithm}

\begin{proof}[Proof of Lemma~\ref{lem:EFTT_is_stable}]    
 From the definition of a stable algorithm (Definition \ref{def:stable}), it suffices to show $\lim_{T\to\infty}\EE_{\bmu}[\omega_2(T)]=1/2$ whenever $\left|\mu_1-\mu_2\right|<\varepsilon/3$. We observe that 
\begin{align*}
    \left|\hat{\mu}_1(2\lfloor \alpha T\rfloor)-\hat{\mu}_2(2\lfloor \alpha T\rfloor)\right|&\le \left|\hat{\mu}_1(2\lfloor \alpha T\rfloor)-\mu_1\right|+\left|\mu_1-\mu_2\right|+\left|\mu_2-\hat{\mu}_2(2\lfloor \alpha T\rfloor)\right|\\
    &\le \frac{\varepsilon}{3}+\left|\hat{\mu}_1(2\lfloor \alpha T\rfloor)-\mu_1\right|+\left|\mu_2-\hat{\mu}_2(2\lfloor \alpha T\rfloor)\right|,
\end{align*}
where the first inequality is from the triangle inequality.
Hence, 
\begin{equation}\label{eq:EFTT}
    \{ \left|\hat{\mu}_1(2\lfloor \alpha T\rfloor)-\hat{\mu}_2(2\lfloor \alpha T\rfloor)\right|>\varepsilon\}\subseteq \left\{\left|\hat{\mu}_1(2\lfloor \alpha T\rfloor)-\mu_1\right|> \frac{\varepsilon}{3}\right\} \cup \left\{\left|\hat{\mu}_2(2\lfloor \alpha T\rfloor)-\mu_2\right|>\frac{\varepsilon}{3}\right\}. 
\end{equation}
Furthermore, the design of Algorithm~\ref{alg:EFTT} yields that if $\left|\hat{\mu}_1(2\lfloor \alpha T\rfloor)-\hat{\mu}_2(2\lfloor \alpha T\rfloor)\right|\le \varepsilon$, then $\left|\omega_2(T)-1/2 \right| \le  1/T$.
This fact together with (\ref{eq:EFTT}) yields that
\begin{align}\label{eq:EFTT1}
\nonumber\PP_{\bmu}\left[\left|\omega_2(T)-\frac{1}{2} \right| > \frac{1}{T}\right]&\le \PP_{\bmu}\left[ \left|\hat{\mu}_1(2\lfloor \alpha T\rfloor)-\hat{\mu}_2(2\lfloor \alpha T\rfloor)\right|>\varepsilon\right]\\
\nonumber    &\le \PP_{\bmu}\left[\left|\hat{\mu}_1(2\lfloor \alpha T\rfloor)-\mu_1\right|> \frac{\varepsilon}{3}\right]+\PP_{\bmu}\left[\left|\hat{\mu}_2(2\lfloor \alpha T\rfloor)-\mu_2\right|> \frac{\varepsilon}{3}\right]\\
    &\le 4\exp\left(\frac{-18\lfloor \alpha T\rfloor}{\varepsilon^2}\right),
\end{align}
where the last inequality is an application of Hoeffding inequality. From (\ref{eq:EFTT1}), we can conclude that $\lim_{T\to \infty}\EE_{\bmu}\left[\omega_2(T)\right]=1/2$ and hence Algorithm~\ref{alg:EFTT} is stable.

\end{proof}

\subsection{The TCSF Algorithm}\label{app:TCSF}

Let $f:[0,1]^2\to [0,1]$ be a continuous function such that $f(a,a)=1/2$ for all $a$. We propose two versions of the TCSF algorithm, one randomized and one de-randomized. Their pseudo-codes are presented in Algorithms \ref{alg:Rf} and \ref{alg:Df}, respectively.

\begin{algorithm}
\caption{Randomized TCSF}\label{alg:Rf}
\begin{algorithmic}[1]
\STATE {\bf Input:} function $f$ and $\alpha>0$
\STATE Play each arm $\max\{\lfloor \alpha T\rfloor,1\}$ times
\STATE $\tau\leftarrow 2\max\{\lfloor \alpha T\rfloor,1\}$
\FOR {$t=\tau+1,\ldots, T$}
    \STATE $\text{play } A_t \leftarrow\left\{\begin{array}{cl}
            2&\text{ w.p. } f(\hat{\bmu}(t))  \\
            1&\text{ w.p. }1-f(\hat{\bmu}(t))
            \end{array}\right.$
\ENDFOR
\STATE $\hat{\imath} \gets \argmax_{k \in \{1, 2\}} \hat{\mu}_k(T)$ (tie broken arbitrarily)
\STATE {\bf Output:}  $\hat{\imath}$
\end{algorithmic}
\end{algorithm}

\medskip
\begin{algorithm}
\caption{De-randomized TCSF}\label{alg:Df}
\begin{algorithmic}[1]
\STATE {\bf Input:} function $f$ and $\alpha>0$
\STATE Play each arm $\max\{\lfloor \alpha T\rfloor,1\}$ times
\STATE $\tau\leftarrow 2\max\{\lfloor \alpha T\rfloor,1\}$
\FOR {$t=\tau+1,\ldots, T$}
    \STATE $\text{play } A_t \leftarrow\left\{\begin{array}{cl}
            2&\text{ if } \omega_2(t)< f(\hat{\bmu}(t))  \\
            1&\text{ otherwise }
            \end{array}\right.$
\ENDFOR
\STATE $\hat{\imath} \gets \argmax_{k \in \{1, 2\}} \hat{\mu}_k(T)$ (tie broken arbitrarily)
\STATE {\bf Output:}  $\hat{\imath}$
\end{algorithmic}
\end{algorithm}

\medskip
\noindent
In the following, we show that Algorithm~\ref{alg:Rf} (resp. Algorithm~\ref{alg:Df}) is stable in Lemma~\ref{lem:RF} (resp. Lemma~\ref{lem:DF}).
\begin{lemma}\label{lem:RF}
If $f:[0,1]^2\mapsto (0,1)$ be a continuous function satisfying that $f(a,a)=1/2,\,\forall a\in [0,1]$, then Algorithm~\ref{alg:Rf} is stable.
\end{lemma}
\begin{proof}
Thanks to Lemma~\ref{lem:generic_stable} in Appendix~\ref{app:generic_tech}, it suffices to show (\ref{eq:generic_stable_low}) and (\ref{eq:generic_stable_up}). In the following, we prove (\ref{eq:generic_stable_low}), and (\ref{eq:generic_stable_up}) hold in a similar manner. To this aim, we fix $a\in (0,1)$ and $\varepsilon>0$. As $f$ is continuous at $(a,a)$ and $f(a,a)=1/2$, there exists $\eta>0$ such that $\left|f(x,y)-1/2\right|<\varepsilon$ if $\left\|(x,y)-(a,a)\right\|_\infty<\eta$. (\ref{eq:generic_stable_low}) follows provided that we show 
\begin{equation}\label{eq:rf-4}
\lowlim_{T\to\infty}\EE_{\bmu}\left[ \omega_2(T)\right]\ge \frac{1}{2}-\varepsilon,\,    \forall \bmu \in \Lambda\text{ such that }\left\|\bmu-(a,a)\right\|_{\infty}<\frac{\eta}{2}.
\end{equation}
Let $\bmu\in \Lambda$ such that $\left\|\bmu-(a,a)\right\|_{\infty}<\frac{\eta}{2}$ and $T\in \NN$ such that $\alpha T>1$. We observe 
\begin{align}\label{eq:rf-1}
\nonumber    \EE_{\bmu}[\omega_2(T)]&\ge \alpha-\frac{1}{T}+\frac{1}{T}\EE_{\bmu}[\sum_{t=\tau+1}^T \mathbbm{1}\{A_t=2\}]\\
\nonumber     &\ge \alpha-\frac{1}{T}+\frac{1}{T}\EE_{\bmu}[\sum_{t=\tau+1}^T \mathbbm{1}\{A_t=2,\left\|\hat{\bmu}(t)-\bmu\right\|_\infty\le\eta/2\}]\\
    &=\alpha-\frac{1}{T}+\frac{1}{T}\EE_{\bmu}[\sum_{t=\tau+1}^T f(\hat{\bmu}(t))\mathbbm{1}\{\left\|\hat{\bmu}(t)-\bmu\right\|_\infty\le\eta/2\}],
\end{align}
where the last inequality is simply from the algorithm design. Notice that if $\left\|\hat{\bmu}(t)-\bmu\right\|_\infty\le\eta/2$, then $\left\|\hat{\bmu}(t)-(a,a)\right\|_\infty\le\eta$, and hence $f(\hat{\bmu}(t))>1/2-\varepsilon$.
We then derive that
\begin{align}\label{eq:rf-2}
\nonumber     \EE_{\bmu}[\sum_{t=\tau+1}^T f(\hat{\bmu}(t))\mathbbm{1}\{\left\|\hat{\bmu}(t)-\bmu\right\|_\infty\le\eta/2\}]&\ge  (\frac{1}{2}-\varepsilon)\EE_{\bmu}[\sum_{t=\tau+1}^T \mathbbm{1}\{\left\|\hat{\bmu}(t)-\bmu\right\|_\infty\le\eta/2\}]\\
    &=(\frac{1}{2}-\varepsilon) (T-\tau-\EE_{\bmu}[\sum_{t=\tau+1}^T \mathbbm{1}\{\left\|\hat{\bmu}(t)-\bmu\right\|_\infty>\eta/2\}]).
\end{align}
As for each $t> \tau$ and $k\in \{1,2\}$, $N_k(t)\ge \alpha T\ge \alpha (t-\tau)$, an application of Lemma~\ref{lem:concentration} in Appendix~\ref{app:generic_tech} with $H=\{t> \tau\},\,\zeta=\alpha$ and $\delta=\eta/2$ yields that
\begin{equation}\label{eq:rf-3}
    \EE_{\bmu}\left[\sum_{t=\tau+1}^T\mathbbm{1}\{ | \hat\mu_k(t) - \mu_k | >\frac{\eta}{2}\}\right]\le \frac{4}{\alpha\eta^2},\quad \forall k=1,2.
\end{equation}
Using (\ref{eq:rf-1})-(\ref{eq:rf-2})-(\ref{eq:rf-3}), we conclude that 
$$
\EE_{\bmu}\left[ \omega_2(T)\right]\ge \alpha-\frac{1}{T}+(\frac{1}{2}-\varepsilon)\frac{(T-\tau-4/\alpha\eta^2)}{T}\ge \alpha-\frac{1}{T}+(\frac{1}{2}-\varepsilon)\frac{(T-\alpha T+1-4/\alpha\eta^2)}{T},
$$
and (\ref{eq:rf-4}) follows.  
\end{proof}

\medskip
\begin{lemma}\label{lem:DF}
   If $f:[0,1]^2\mapsto (0,1)$ be a continuous function satisfying that $f(a,a)=1/2,\,\forall a\in [0,1]$, then Algorithm~\ref{alg:Df} is stable.
\end{lemma}
\begin{proof}
Thanks to Lemma~\ref{lem:generic_stable} in Appendix~\ref{app:generic_tech}, it suffices to show (\ref{eq:generic_stable_low}) and (\ref{eq:generic_stable_up}). In the following, we prove (\ref{eq:generic_stable_up}), and (\ref{eq:generic_stable_low}) hold in a similar manner. To this aim, we fix $a\in (0,1)$ and $\varepsilon>0$. As $f$ is continuous at $(a,a)$ and $f(a,a)=1/2$, there exists $\eta>0$ such that $\left|f(x,y)-1/2\right|<\varepsilon$ if $\left\|(x,y)-(a,a)\right\|_\infty<\eta$. (\ref{eq:generic_stable_up}) follows as long as we show 
\begin{equation}\label{eq:df-1}
\uplim_{T\to\infty}\EE_{\bmu}\left[ \omega_2(T)\right]\le \frac{1}{2}+\varepsilon,\,    \forall \bmu \in \Lambda\text{ such that }\left\|\bmu-(a,a)\right\|_{\infty}<\frac{\eta}{2}.
\end{equation}
Let $\bmu\in \Lambda$ such that $\left\|\bmu-(a,a)\right\|_{\infty}<\frac{\eta}{2}$ and $T\in \NN$ such that $\alpha T>1$. From the algorithm design, we deduce that
\begin{align}\label{eq:df-2}
\nonumber    \EE_{\bmu}[\omega_2(T)]&\le \alpha+\frac{1}{T}+\frac{1}{T}\EE_{\bmu}[\sum_{t=\tau+1}^T \mathbbm{1}\{\omega_2(t)\le f(\hat{\bmu}(t))\}]\\
     &\le \alpha+\frac{1}{T}+\frac{1}{T}\left(\EE_{\bmu}[\sum_{t=\tau+1}^T \mathbbm{1}\{\omega_2(t)\le f(\hat{\bmu}(t)),\left\|\hat{\bmu}(t)-\bmu\right\|_\infty\le\frac{\eta}{2}\}]+\EE_{\bmu}[\sum_{t=\tau+1}^T\mathbbm{1}\{\left\|\hat{\bmu}(t)-\bmu\right\|_\infty>\frac{\eta}{2}\}]\right).
\end{align}
As for each $t> \tau$ and $k\in \{1,2\}$, $N_k(t)\ge \alpha T\ge \alpha (t-\tau)$, an application of Lemma~\ref{lem:concentration} in Appendix~\ref{app:generic_tech} with $H=\{t> \tau\},\,\zeta=\alpha$ and $\delta=\eta/2$ yields that
\begin{equation}\label{eq:df-3}
    \EE_{\bmu}\left[\sum_{t=\tau+1}^T\mathbbm{1}\{ | \hat\mu_k(t) - \mu_k | >\frac{\eta}{2}\}\right]\le \frac{4}{\alpha\eta^2},\quad \forall k=1,2.
\end{equation}
Next we observe that $\left\|\hat{\bmu}(t)-\bmu\right\|_\infty\le\frac{\eta}{2}$ implies that $\left\|\hat{\bmu}(t)-(a,a)\right\|_\infty\le\eta$, and $f(\hat{\bmu}(t))<1/2+\varepsilon$ thanks to the triangle inequality. Thus, the third term in (\ref{eq:df-2}) is bound as:
\begin{align}\label{eq:df-4}
 \nonumber    \EE_{\bmu}[\sum_{t=\tau+1}^T \mathbbm{1}\{\omega_2(t)\le f(\hat{\bmu}(t)),\left\|\hat{\bmu}(t)-\bmu\right\|_\infty\le\frac{\eta}{2}\}]&\le \EE_{\bmu}[\sum_{t=\tau+1}^T \mathbbm{1}\{\omega_2(t)\le \frac{1}{2}+\varepsilon\}]\\
 \nonumber   &=\EE_{\bmu}[\sum_{t=\tau+1}^T \mathbbm{1}\{N_2(t)\le t(\frac{1}{2}+\varepsilon)\}] \\
 \nonumber    &\le \EE_{\bmu}[\sum_{t=\tau+1}^T \mathbbm{1}\{N_2(t)\le T(\frac{1}{2}+\varepsilon)\}] \\
 \nonumber    &\le T(\frac{1}{2}+\varepsilon)-\lfloor\alpha T\rfloor\\
     &\le T(\frac{1}{2}+\varepsilon)-\alpha T+1.
\end{align}
By (\ref{eq:df-3})-(\ref{eq:df-4}), we derive (\ref{eq:df-2}) is bounded by 
$$
\alpha+\frac{1}{T}+\frac{1}{T}\left( T(\frac{1}{2}+\varepsilon)-\alpha T+1 +\frac{4}{\alpha\eta^2}  \right).
$$
By taking the limit superior on the above upper bound, we get (\ref{eq:df-1}).
\end{proof}

\subsection{Technical Lemmas}\label{app:generic_tech}

\begin{lemma}\label{lem:generic_stable}
    Suppose that an algorithm satisfies 
    \begin{equation}\label{eq:generic_stable_low}
            \lim_{(\mu_1,\mu_2)\to (a,a)}\lowlim_{T\to\infty} \EE_{\bmu}[\omega_2(T)]=\frac{1}{2},\quad \forall a\in (0,1),
    \end{equation}
    and
     \begin{equation}\label{eq:generic_stable_up}
            \lim_{(\mu_1,\mu_2)\to (a,a)}\uplim_{T\to\infty} \EE_{\bmu}[\omega_2(T)]=\frac{1}{2},\quad \forall a\in (0,1).
    \end{equation}
    Then it is a stable algorithm.
\end{lemma}%
\begin{proof}
    Let $a\in (0,1)$, we show (A) in Definition~\ref{def:stable} holds, and (B) follows similarly. Consider a sequence $\{\bl^{(n)}\}_{n\in\NN}$ defined as: $\lambda^{(n)}_1=a+\frac{1-a}{2n}$ and $\lambda^{(n)}_2=a-\frac{a}{2n}$ for all $n\in \NN$. The assumption (\ref{eq:generic_stable_low}) implies that 
    $$
\lim_{n\to\infty}\lowlim_{T\to\infty} \EE_{\bl^{(n)}}[\omega_2(T)]=\frac{1}{2}.
    $$
    On the other hand, the assumption (\ref{eq:generic_stable_up}) implies that 
 $$
   \lim_{n\to\infty}\uplim_{T\to\infty} \EE_{\bl^{(n)}}[\omega_2(T)]=\frac{1}{2}.
    $$
    Thus, (A) is satisfied with the above sequence $\{\bl^{(n)}\}_{n\in\NN}$.
\end{proof}

\begin{lemma}[\cite{combes2014unimodal}]\label{lem:concentration}
    Let $\zeta> 0$ and $H \subset \mathbb{N}$ be a (random) set of rounds such that $\{t\in H\}$ is ${\cal F}_{t-1}$-measurable for all $t\geq 1$. Furthermore, we assume for each $t\in H$, we have $N_k(t)\ge \zeta \sum_{s=1}^{t} \mathbbm{1}_{\{s\in H\}}$. Then for all $\delta>0$,
\begin{equation*}\label{eq:ineq1}
\mathbb{E}_{\bmu}\left[ \sum_{t \geq 1} \mathbbm{1}{\{ t \in H , | \hat\mu_k(t) - \mu_k | > \delta \} }\right]  \leq  \frac{1}{\zeta\delta^2}.
\end{equation*}
\end{lemma}

\subsection{Numerical Experiments}\label{app:num}

We illustrate the performance of the ETT algorithm with $\alpha=1/4$ and different thresholds $\varepsilon$, and compare it to that of the uniform sampling algorithm and to that of an Oracle algorithm that selects arms using optimal exploration rate $x^\star(\bmu)=\argmax_x g(x,\bmu)$. 
We consider the instance: $\bmu = (0.0005, 0.0001)$. For this instance, the optimal budget allocation is approximately $x^*(\bmu) \approx0.43434$.

We first examine how the algorithms behave when the sampling budget varies. Figure~\ref{fig:exp_vary_budget} illustrates the estimated error probabilities as the budget changes from $T=6000$ to $T=40000$. The error probabilities are derived from $40000$ trials for each setting and algorithm. In all budget scenarios, the Oracle algorithm outperforms the others, while ETT performs comparably or worse than the uniform sampling algorithm. This observation aligns with our Theorem~\ref{thm:stable}.

We then investigate the sensitivity of ETT to the input value $\varepsilon$. Figure~\ref{fig:exp_vary_input} displays the error probability with a fixed budget of $T=20000$ and varying $\varepsilon$ from $0$ to $0.0008$. The error probability is again determined from $40000$ trials for each setting and algorithm. Regardless of $\varepsilon$, the performance of ETT  is similar to or worse than that of the uniform sampling algorithm, further supporting our Theorem~\ref{thm:stable}.
Given that $\mu_1-\mu_2 = 0.0004$, the relatively low performance of ETT with $\varepsilon<0.0004$ compared to that of the uniform sampling algorithm suggests that relying less on the estimated optimal allocation could yield better results.

\begin{figure}[ht]
\centering
\includegraphics[width=0.7\textwidth]{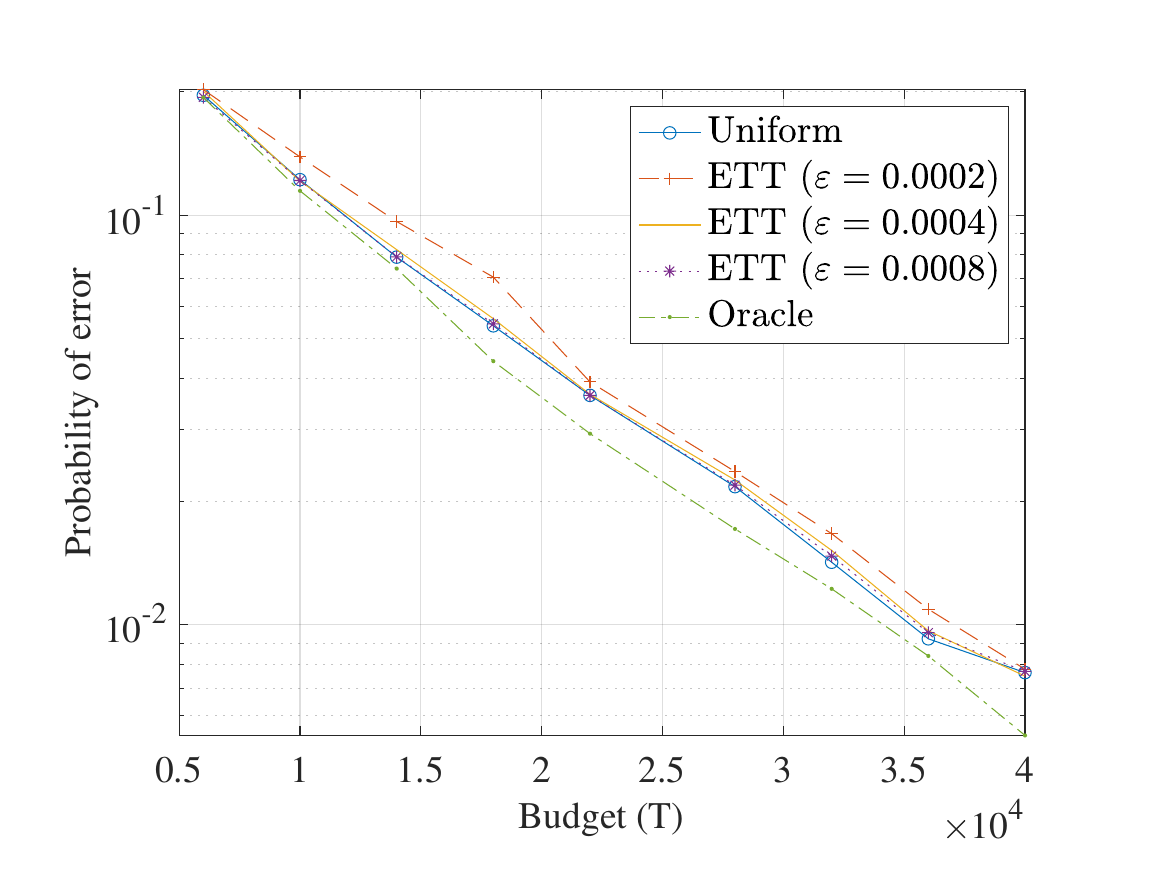}
\caption{Error probability comparison across algorithms for varying sample budgets ($T=6000$ to $T=40000$). Derived from $40000$ trials for each setting and algorithm.}
\label{fig:exp_vary_budget}
\end{figure}

\begin{figure}[ht]
\centering
\includegraphics[width=0.7\textwidth]{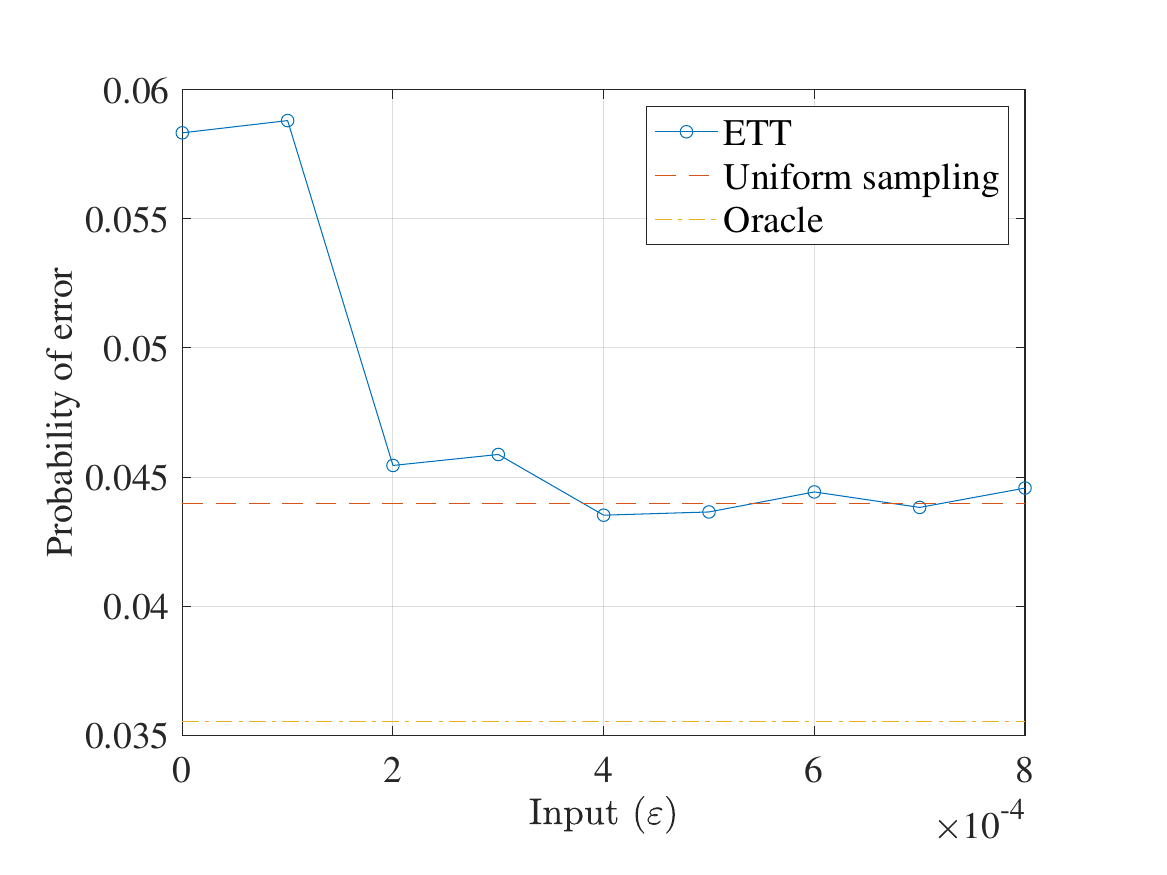}
\caption{Error probability comparison for varying ETT threshold inputs ($\varepsilon=0$ to $\varepsilon=0.0008$). Derived from $40000$ trials for each setting and algorithm.}
\label{fig:exp_vary_input}
\end{figure}

\end{document}